\documentclass{article}
\def\comments{0} 
\newif\ifshort 
\newif\iflong 
\ifshort \longfalse \else \longtrue \fi

\usepackage{comment}
\ifnum\comments=1
     \newcommand{\ha}[1]{\textcolor{blue}{\textbf{HA:} {#1}}}
    \newcommand{\lz}[1]{\textcolor{magenta}{\textbf{LZ:} {#1}}}
    \newcommand{\ju}[1]{\textcolor{red}{\textbf{JU:} {#1}}}
\else
    \newcommand{\ha}[1]{}
    \newcommand{\lz}[1]{}
    \newcommand{\ju}[1]{}
\fi

\usepackage{macros}
\bluehyperref
\usepackage[capitalize,noabbrev]{cleveref}

\ifshort
    
    \usepackage{icml2023}
    \theoremstyle{plain}
    \newtheorem{theorem}{Theorem}[section]
    \newtheorem{proposition}[theorem]{Proposition}
    
    \newtheorem{corollary}[theorem]{Corollary}
    \theoremstyle{definition}
    \newtheorem{definition}[theorem]{Definition}
    
    \theoremstyle{remark}
    \newtheorem{remark}[theorem]{Remark}
    \icmltitlerunning{From Robustness to Privacy and Back}
\else
    \usepackage[utf8]{inputenc}
    \usepackage[margin=1in]{geometry}
    \usepackage[noend]{algorithmic}
    \usepackage{algorithm}
    \usepackage[style=alphabetic,backend=bibtex,maxnames=12,maxbibnames=10,maxcitenames=10,minalphanames=3,maxalphanames=5,giveninits=true,doi=false,url=true]{biblatex}
    \bibliography{bib}
    \newcommand*{\citet}[1]{\AtNextCite{\AtEachCitekey{\defcounter{maxnames}{2}}} \textcite{#1}}
    \newcommand*{\citep}[1]{\cite{#1}}
    \newcommand{\citeyearpar}[1]{\cite{#1}}
\fi

\usepackage{crossreftools}
\iflong
    \title{From Robustness to Privacy and Back}
    \date{\today}
    \author{Hilal Asi\thanks{Apple. \url{hilal.asi94@gmail.com}} \and Jonathan Ullman\thanks{Khoury College of Computer Science, Northeastern University. Supported by NSF awards CCF-1750640, CNS-1816028, and CNS-2120603.  \url{jullman@ccs.neu.edu} } \and Lydia Zakynthinou\thanks{Khoury College of Computer Science, Northeastern University. Supported by NSF awards CCF-1750640, CNS-1816028, and CNS-2120603, and a Facebook Fellowship. \url{zakynthinou.l@northeastern.edu}.}}
\fi
\begin{document}
\ifshort
\else
     \maketitle
\fi

\begin{abstract}
We study the relationship between two desiderata of algorithms in statistical inference and machine learning---differential privacy and robustness to adversarial data corruptions.  Their conceptual similarity was first observed by Dwork and Lei (STOC 2009), who observed that private algorithms satisfy robustness, and gave a general method for converting robust algorithms to private ones.  However, all general methods for transforming robust algorithms into private ones lead to suboptimal error rates.  Our work gives the first black-box transformation that converts any adversarially robust algorithm into one that satisfies pure differential privacy.  Moreover, we show that for any low-dimensional estimation task, applying our transformation to an optimal robust estimator results in an optimal private estimator.  Thus, we conclude that for any low-dimensional task, the optimal error rate for $\diffp$-differentially private estimators is essentially the same as the optimal error rate for estimators that are robust to adversarially corrupting $1/\diffp$ training samples.  We apply our transformation to obtain new optimal private estimators for several high-dimensional tasks, including Gaussian (sparse) linear regression and PCA.  Finally, we present an extension of our transformation that leads to approximate differentially private algorithms whose error does not depend on the range of the output space, which is impossible under pure differential privacy.
\end{abstract}

\section{Introduction}\label{sec:intro}
Both \emph{differential privacy} and \emph{robustness} are desirable properties for machine learning or statistical algorithms, and there are extensive, mostly separate bodies of research on each of these properties.

\emph{Differential privacy (DP)} was proposed by Dwork, McSherry, Nissim, and Smith~\citep{DworkMNS06} as a rigorous formalization of what it means for an algorithm to guarantee individual privacy, and has been widely deployed in both industry~\citep{ErlingssonPK14,Bittau+17,Apple17,TestuggineM20,WilsonZLDSG20, Rogers+20} and government~\citep{Haney+17, Abowd18, AbowdACKLORZ22} applications. 
Informally, a DP algorithm ensures that no adversary who observes the algorithm's output can learn much more about an individual in the dataset than they could if that individual's data had been excluded.  Formally, a randomized algorithm $\A$ satisfies $\eps$-DP if for every dataset $\Ds$, and every dataset $\Dsc$ that differs on one, or a small number of entries, the distributions $\A(\Ds)$ and $\A(\Dsc)$ are $\eps$-close in a precise sense (see~\Cref{def:dp}), where the privacy guarantee becomes stronger as $\eps$ becomes smaller.


\emph{Robustness}, which was first systematically studied by Tukey and Huber in the 1960s~\citep{Tukey60, Huber65}, formalizes algorithms that perform well under data corruptions or model misspecifications.  Specifically, we consider a dataset $\Ds$ that is drawn iid from some well behaved distribution $P$, and allow an adversary to produce a dataset $\Dsc$ that differs from $\Ds$ in a $\tau$ fraction of its entries.  An algorithm $\A$ is $\tau$-robust if the distance $\norm{\A(\Ds) - \A(\Dsc)}$ is typically small in some particular error norm, where the robustness guarantee becomes stronger as $\tau$ becomes larger.

Although these two conditions have entirely different motivations, they are both notions of what it means for an algorithm to be insensitive to small changes in its input, which was first observed in the influential work of~\citet{DworkL09}.  But even once we recognize their similarity, there are substantial technical differences.  While differentially private algorithms are robust, DP is a more stringent requirement in a few ways: First, DP is \emph{worst case}, meaning the algorithm $\A$ must be insensitive in a neighborhood around \emph{every} dataset $\Ds$, whereas a robust algorithm only needs to be insensitive in the \emph{average case} around datasets $\Ds$ drawn from well behaved distributions $P$.  Second, DP requires that $\A(\Ds)$ an $\A(\Dsc)$ be close \emph{as probability distributions} in a strong sense, whereas robustness only requires the \emph{distance between outputs} $\A(\Ds)$ and $\A(\Dsc)$ to be small.  On the other hand, since DP is harder to satisfy, \iflong differentially private \else DP \fi algorithms are typically insensitive to changes in a \emph{small number} of inputs, whereas robust algorithms can often be insensitive to changes in a \emph{small constant fraction} of inputs.

Although, differential privacy is stronger than robustness, Dwork and Lei~\citep{DworkL09} designed a method, called \emph{propose-test-release (PTR)}, which can be used to turn any accurate robust algorithm for a statistical estimation task into an accurate differentially private algorithm for the same task.  However, this generic transformation typically does not lead to optimal algorithms for most specific tasks of interest.  Nonetheless, there has been a recent flurry of works in differentially private statistics that use robust estimators as \emph{inspiration} for differentially private estimators (see Related Work for a summary of this line of work). These works use a variety of methods for upgrading robust estimators to differentially private ones, and each of these methods is tailored to a specific task or set of tasks.

In this paper we demonstrate that there is, in fact, a \emph{black-box} way to transform robust estimators into private estimators that provably gives optimal error rates for low-dimensional tasks, and often leads to optimal error rates for many high-dimensional tasks.



\subsection{Our Contributions}

In this section we give a high-level overview of our main contributions---black-box transformation from robust to private estimators, optimality results for our transformations for low-dimensional estimation tasks, and applications of our transformations to high-dimensional estimation tasks.

\mypar{From robustness to privacy via inverse-sensitivity.}
Our first main contribution is a black-box transformation that takes a robust algorithm for any statistical task and converts it into an $\eps$-differentially private algorithm for the same task with comparable accuracy.  Intuitively, since robust estimators are insensitive to changing $n\tau$ inputs on a dataset of size $n$, and private estimators are insensitive to changing a $1/\eps$ total inputs, the accuracy of the private estimator will be related to the accuracy of the robust estimator when a $\tau \approx 1/n\eps$ fraction of inputs can be corrupted.

Our transformation applies in the standard setting of statistical estimation: assume that there exists a distribution $P$ over domain $\domain\subseteq \R^d$ and $\Ds=(S_1,\ldots, S_n)$ is a dataset consisting of $n$ examples drawn independently from $P$, that is, $\forall i \in [n]$, $\smash{S_i \simiid P}$. We wish to estimate a parameter $\mu(P)$ (e.g. the mean $\mu = \E_{s \sim P}[s]$) of distribution $P$. The error of an algorithm $\A$ for this task is measured with respect to a norm $\norm{\cdot}$, i.e., $\norm{\A(S)-\mu}$. 

We use the following short-hand to denote the accuracy guarantees of $\tau$-robust and $\diffp$-private estimators of a statistic $\mu$ for a distribution $P$. 
\begin{definition}[$(\tau,\beta, \alpha)$-robust estimator]\label{def:robustness}
    Let $\A$ be a (possibly randomized) algorithm for the estimation of statistic $\mu$. We say that $\A$ is a $(\tau,\beta,\alpha)$-robust estimator for distribution $P$, 
    if with probability $1-\beta$ over dataset $\Ds\simiid P^n$ (and possibly the randomness of the algorithm), for all $\Dsc$ that differ in at most $n\tau$ points from $\Ds$, 
    we have that 
    \begin{equation*}
        \norm{\A(\Dsc) - \mu(P)} \le \alpha.
    \end{equation*}
\end{definition}

\begin{definition}[$(\diffp,\beta,\alpha)$-private estimator]
    Let $\A$ be a (possibly randomized) algorithm for the estimation of statistic $\mu$. We say that $\A$ is an $(\diffp,\beta, \alpha)$-private estimator for distribution $P$, if $\A$ is $\diffp$-DP (\Cref{def:dp}) and with probability $1-\beta$ over $\Ds\simiid P^n$ (and possibly the randomness of the algorithm), we have that
    \begin{equation*}
        \norm{\A(\Ds) - \mu(P)} \le \alpha.
    \end{equation*}
\end{definition}
We may refer to such algorithms as $(\tau,\beta,\alpha)$-robust and $(\diffp,\beta,\alpha)$-private for brevity. 

Our main theorem shows that any robust algorithm can be transformed into an $\diffp$-DP algorithm with the same accuracy guarantees up to constants, as long as the fraction of corruptions $\tau\approx \frac{d\log(R)}{n\diffp}$, where $R$ is the diameter of the range of the robust algorithm. 
\begin{theorem}[Informal, \Cref{thm:high-dim-rob-to-priv}]\label{thm:high-dim-rob-to-priv-informal}
Let $n\geq 1, \eps\in(0,1)$. Let $P$ be a distribution over $\domain\subseteq \R^d$. 
Let $\A_\mathrm{rob}: \domain^n \to \{t\in\R^d: \norm{t}\leq R\}$ be any $(\tau,\beta,\alpha)$-robust algorithm for the statistic $\mu(P)$, where $$\tau \gtrsim \frac{d\log(R)+\log(1/\beta)}{n\diffp}.$$ 
Then there exists an $(\diffp,O(\beta),O(\alpha))$-private algorithm $\A_{\mathrm{priv}}$ for the statistic $\mu(P)$.  The notation $\gtrsim$ above hides constants and logarithmic factors in $1/\alpha$.
\end{theorem}

The main idea behind our transformation is to use the \emph{inverse-sensitivity mechanism}~\citep{AsiD20a}.  At a high level, for a deterministic algorithm $\A_{\mathrm{rob}}$, the inverse-sensitivity mechanism outputs $t$ with probability (or density) proportional to
\begin{equation*}
  \Pr[\mechinv(\Ds;\A_{\mathrm{rob}}) = t] \propto e^{- \invmodcont_{\A_{\mathrm{rob}}}(\Ds;t) \cdot \diffp/ 2}
\end{equation*}
where $\invmodcont_{\A_{\mathrm{rob}}}(\Ds;t)$ is the minimum number of entries of $\Ds$ that would have to be corrupted to obtain a dataset $\Dsc$ with $\A_{\mathrm{rob}}(\Dsc) = t$.  This mechanism is an instance of the exponential mechanism~\citep{McSherryT07}, and to the best of our knowledge the idea to use $\invmodcont$ as a score function first appeared in~\citep{JohnsonS13} for applications in genomic data, and its general accuracy guarantees were first studied systematically in~\citep{AsiD20a, AsiD20b}.  A standard analysis shows that this estimator will output $t$ for which $\invmodcont_{\A_{\mathrm{rob}}}(\Ds;t)$ is small, and we can relate the accuracy of such a $t$ to the robustness of the estimator $\A_{\mathrm{rob}}$ on corruptions of the sample $\Ds$.

We note that our transformation only preserves the accuracy of the robust mechanism, but not its computational efficiency, and an interesting open question is whether one can get a fully black-box, efficiency-preserving transformation from robustness to privacy (see the concurrent and independent work of~\citep{HopkinsKMN22} for some progress on this question).


We also define and analyze an extension of this transformation, which is based on the restricted exponential mechanism of Brown et al.~\citep{BrownGSUZ21}, that avoids the dependence on $R$ that appears in~\Cref{thm:high-dim-rob-to-priv-informal} above, by relaxing the privacy definition to approximate DP. 


\mypar{An equivalence between private and robust estimation.} 
We prove that, up to the factor of $d$ in \Cref{thm:high-dim-rob-to-priv-informal}, our transformation is optimal, and in particular is optimal for low-dimensional tasks when $d$ is a constant.  That is, for any low-dimensional task, if $\A_{\mathrm{rob}}$ is an optimal robust estimator, then our transformation of $\A_{\mathrm{rob}}$ is an optimal private estimator.  This is the first result to show a general conversion from robust estimators to \emph{optimal} private estimators.

A consequence of this result is an equivalence between robust and private estimation in low dimensions, which shows that the optimal minimax error rates for $\diffp$-DP estimation and for $\tau$-robust estimation for $\tau \approx 1/n\diffp$ are essentially the same.  Specifically, for a given statistic $\mu$ and a family of distributions $\P$ over domain $\domain$, we define the minimax error of estimating $\mu$ under $\P$ for private and robust algorithms as follows. Having fixed $\beta$, $\tau$, and $\diffp$, the $(\tau,\beta)$-robust minimax error under $\P$ is
\begin{equation*}
    \alpha_{\mathrm{rob}}(\tau,\beta) = \inf_{\alpha}  \{ \alpha: \text{$\exists (\tau,\beta,\alpha)$-robust algorithm $\forall P\in\P$}\},
\end{equation*}
and the $(\diffp,\beta)$-private minimax error under $\P$ is
\begin{equation*}
    \alpha_{\mathrm{priv}}(\diffp,\beta) = \inf_{\alpha}  \{ \alpha: \text{$\exists (\diffp,\beta,\alpha)$-private algorithm $\forall P\in\P$}\}. 
\end{equation*}

\begin{theorem}[Informal, \Cref{cor:eqv}]
\label{cor:eqv-inf}
    Let $\P$ be a family of distributions over $\R$ and let $\mu$ be a $1$-dimensional statistic where $|\mu(P)| \le 1$ for all $ P\in \P$. Suppose $\accr(\tau,\beta)$ is a continuous function of $\beta$ for all $\tau \le 1/2$.
    Let $n > 1$, $\diffp =\omega(\log(n)/n)$ and $\tau = \Theta(\log(n)/n\diffp)$. 
    Suppose there exists constant $c$ such that the non-private error $\accr(0,\beta)\geq \frac{1}{n^c}$ for any $\beta \le 1/4$. Then there are constants $c_1 \ge c_2 >0$ such that for $\beta_p = 1/n^{c_1}$ and $\beta_r = 1/n^{c_2}$, 
    \begin{equation*}
        \accp(\diffp,\beta_p) = \Theta \left( \accr(\tau,\beta_r) \right).
    \end{equation*}
\end{theorem}
The above theorem extends to any $d$-dimensional problem with a weaker conclusion roughly $\accp(\diffp,\beta_p) = \Theta \left(d \cdot \accr(\tau,\beta_r) \right)$, so in particular we obtain the same equivalence for any problem in constant dimension.

The theorem follows from the folklore observation that differentially private estimators are also robust (with $\tau \approx 1/\eps n$).  Thus, if we have an optimal differentially private algorithm $\A_{\mathrm{priv}}$ we can use $\A_{\mathrm{priv}}$ itself as the robust estimator.  Thus, we can instantiate the inverse-sensitivity mechanism using $\A_{\mathrm{priv}}$ as the robust estimator and apply the inverse-sensitivity mechanism to $\A_{\mathrm{priv}}$ to obtain a new private mechanism with similar accuracy.  Although this transformation would be a circular way to produce a private estimator, the argument shows that one can always obtain an optimal private estimator by instantiating our transformation with an optimal robust estimator.


\mypar{Applications to high-dimensional private estimation.}
Although our general optimality result only applies to low-dimensional problems, we show that our transformation often yields optimal error bounds for several high-dimensional tasks as well. By instantiating~\Cref{thm:high-dim-rob-to-priv-informal} with existing algorithms for robust estimation, we give $\eps$-DP algorithms for the same tasks. At high level,~\Cref{thm:high-dim-rob-to-priv-informal} says that if there is a robust algorithm with accuracy $\alpha(\tau)$ as a function of the fraction of corruptions, then there is an $\diffp$-DP algorithm with accuracy $\alpha(D/n\diffp)$, where $D$ is the dimension of the parameter we aim to estimate. Since usually $\alpha(\tau)=\tilde{O}(\tau)$, this implies that the error due to privacy is $\tilde{O}(D/n\diffp)$. 
In particular, we apply our theorem to give pure differentially private algorithms for Gaussian (sparse and non-sparse) linear regression and subgaussian PCA, which, to the best of our knowledge, are the first optimal algorithms for these tasks satisfying pure (rather than approximate) differential privacy.

\subsection{Related Work}

\mypar{General transformations between robust and private algorithms.}
\citet{DworkL09} were the first to observe the intuitive connection between differential privacy and robust statistics.  That work also introduced a generic framework for differentially private algorithms called \emph{propose-test-release (PTR)} that can be used to transform any robust estimator into an approximately DP estimator.  However, compared to our optimal transformation, the error of the resulting private algorithm will be larger by a factor of $\approx 1/\eps$.  An even earlier work of Nissim, Raskhodnikova, and Smith~\citep{NissimRS07} presented a framework called \emph{smooth sensitivity} that can be used to obtain a similar transformation from robust to pure DP estimators, again losing a factor of $\approx 1/\eps$ compared to our transformation.

In the other direction,~\citet{DworkL09} also observed that differentially private algorithms are robust with certain parameters.  However, private estimators are mostly studied in a regime where $1/\eps = o(n)$, so they do not give robustness to corrupting a constant fraction of inputs, which is the most commonly studied regime for robust estimation.  More recently,~\citet{GeorgievH22} observed that private algorithms with sufficiently small failure probability and privacy parameter are robust to corrupting a constant fraction of inputs, and use this fact to give evidence of computational hardness of certain private estimation tasks.

\mypar{Private estimators inspired by robust estimators.}
Although prior black-box transformations from robust to private estimators give suboptimal error rates, many optimal private algorithms are nonetheless inspired by methods from robust statistics, albeit with task-specific analyses~\citep{KamathLSU19, BunS19, Avella-MedinaB19, Avella-MedinaB20, KamathSU20, LiuKK021, BrownGSUZ21, GhaziKMN21, LiuKO22, TsfadiaCKMS22, HopkinsKM22, AshtianiL22, KothariMV22, AlabiKTVZ22, GeorgievH22}.  Their algorithms often leverage the structure of specific robust estimators such as medians, high-dimensional generalizations of the median, trimmed-means, or sum-of-squared-based certificates of robustness.

An elegant work of~\citet{LiuKO22} proposed a generalization of PTR that can be used to give near-optimal approximate differentially private estimators for many tasks.  Although the framework is fairly general, their analysis relies on specific properties of the estimation tasks rather than merely the existence of a robust estimator.

Another line of work is that of algorithms that specifically aim to satisfy optimal robustness and privacy guarantees simultaneously for high-dimensional problems~\cite{LiuKK021, GhaziKMN21, AlabiKTVZ22}.

\mypar{Concurrent work by\citet{HopkinsKMN22}.} In concurrent and independent work,\citet{HopkinsKMN22} proposed the same black-box transformation from robust to DP algorithms based on the smooth-inverse-sensitivity mechanism.  In contrast to our work, they also show that in some cases their method can be implemented in a computationally efficient way by instantiating the smooth-inverse-sensitivity mechanism with robust estimators based on the sum-of-squares paradigm.  In particular they construct computationally efficient pure DP algorithms for estimating a Gaussian distribution with optimal error.  In contrast to their work, we demonstrate that our transformation gives optimal error for low-dimensional problems and establishes a tight connection between privacy and robustness for these problems. 

\subsection{Organization}
In~\Cref{sec:background}, we provide background on DP and the inverse-sensitivity mechanism. In~\Cref{sec:transformations}, we present and prove the guarantees of our transformation from robust to pure DP algorithms (and vice versa, for completeness). 
In~\Cref{sec:implications}, we show the optimality of our transformation for low-dimensional tasks and the equivalence of robustness and privacy for those tasks. In~\Cref{sec:approxDP}, we extend our transformation to convert robust to approximate DP algorithms.
\iflong 
    In~\Cref{sec:applications}, we show that our main theorem gives us pure DP algorithms with near-optimal error for several statistical tasks such as Gaussian mean and covariance estimation, (sparse) linear regression, and subgaussian PCA. We conclude with a discussion and questions for future work in~\Cref{sec:future}.
\else
    In~\Cref{sec:applications}, we show that our transformation gives us pure DP algorithms with near-optimal error for PCA for subgaussian data, deferring more applications to~\Cref{app:moreapplications}. 
\fi
\section{Preliminaries and Background}\label{sec:background}

\paragraph{Additional Notation}
For a finite set $\range$, we denote its cardinality by $\card(\range)$. For any (continuous or discrete) set $\range$, we denote its diameter by $\diam(\range)=\sup_{s,t \in \range} \norm{s-t}$, where the choice of norm will be clear from context. We denote by $\dham{}{}$ the Hamming distance between two vectors or datasets. We denote by $\ball^d(v,R)$ the $d$-dimensional ball with radius $R$ and center $v$ (with respect to some norm $\norm{\cdot}$). We also let $\ball^d(R) = \ball^d(0^d,R)$ and omit $d$ when it is clear from context.

\subsection{Differential Privacy}
Let $\Ds, \Ds'\in \domain^n$ be two data sets of size $n$. We say that $\Ds, \Ds'$ are \textit{neighboring data sets} if $\dham(\Ds,\Ds')\le 1$. 
Differentially private algorithms have indistinguishable output distributions on neighboring data sets. 
\begin{definition}[Differential Privacy \cite{DworkMNS06}] \label{def:dp}
A (possibly randomized) algorithm $\A\colon\domain^n \to \range$ is  {\em $(\eps,\delta)$-differentially private} (DP) if for all neighboring datasets $\Ds, \Ds'$ and any measurable output space $T\subseteq \range$ we have $$\Pr[\A(\Ds)\in T]\leq e^{\eps} \Pr[\A(\Ds')\in T]+\delta.$$
We say algorithm $\A$ satisfies {\em pure} DP if it satisfies the definition for $\delta=0$, which we denote by $\eps$-DP. Otherwise, we say it satisfies \emph{approximate} DP.
\end{definition}

The exponential mechanism~\cite{McSherryT07} is a ubiquitous building block for constructing DP algorithms. The inverse-sensitivity mechanism, on which our transformation is based, is an instantiation of this mechanism. 
\begin{definition}[Exponential Mechanism,~\cite{McSherryT07}]\label{def:expmech}
Let input data set $\Ds\in\domain^n$, range $\range$, and score function $q: \domain^n\times \range \rightarrow \R$ with sensitivity $\Delta_q=\max\limits_{t\in\range} \max\limits_{\Dsc: \dham(\Dsc, \Ds)\leq 1} |q(\Ds;t)-q(\Ds';t)|.$
The {\em exponential mechanism} selects and outputs an element $t\in\range$ with probability 
$\pi_{\Ds}(t) \propto e^{\left({\eps \cdot q(\Ds;t)}/{2\Delta_q}\right)}$. 
The exponential mechanism is $\eps$-DP.
\end{definition}


\subsection{Inverse-sensitivity Mechanism}
Let $f:\domain^n \to \range$ be a (deterministic) algorithm that we aim to compute on dataset $\Ds$.
Define the path-length function
\begin{equation*}
  \label{eqn:inverse-ls}
  \invmodcont_\func(\Ds; t) \defeq
  \inf_{\ds'} \left\{ \dham(\Ds, \Ds')
  \mid \func(\Ds') = t \right\},
\end{equation*}
which is the minimum number of points in $\Ds$ that need to be replaced so that the value of function $\func$ on the modified input is $t$.
Given black-box access to function $f$, the \emph{inverse-sensitivity mechanism} with input dataset $\Ds$, denoted by $\mechinv(\Ds;\func)$, is then defined as follows: the probability that $\mechinv(\Ds;\func)$ returns $t\in\range$ is
\begin{equation*}
  \label{mech:discrete}
  \Pr[\mechinv(\Ds;\func) = t]
  \defeq \frac{e^{-\invmodcont_\func(\Ds;t) \diffp/2}}{
    \sum_{s \in \range} e^{-\invmodcont_\func(\Ds;s) \diffp/2}}.
\end{equation*}

The error of this mechanism in some norm $\norm{\cdot}$ depends on the \emph{local modulus of continuity} of a function $f : \domain^n \to \range$
at $\Ds \in \domain^n$ with respect to $\norm{\cdot}$, defined by
\begin{equation*}
  \label{eqn:modcont-def}
  \modcont_\func (\Ds;K) = \sup_{\Ds' \in \domain^n} \left\{
  \norm{\func(\Ds)-\func(\Ds')} : \dham(\Ds,\Ds') \le K \right\}.
\end{equation*}

For a finite set $\range$, the inverse-sensitivity mechanism has the following guarantees. 
\begin{theorem}[Discrete functions, Th.3~\citep{AsiD20a}]
  \label{thm:upper-bound-discrete}
  Let $\func: \domain^n \to \range$ and $D=\diam(\range)$.
  Then for any $\Ds \in \domain^n$ and $\beta>0$, with probability at least $1-\beta$, the inverse-sensitivity mechanism has error
  \begin{equation*}
  \norm{\mechinv(\Ds)-\func(\Ds)}
    \le \modcont_\func\left(\Ds; \frac{2}{\diffp}
    \log\frac{2 D \card(\range)}{\beta \diffp}
    \right).
  \end{equation*}
\end{theorem}

For continuous functions $\func: \domain^n \to \R^d$, which is the main setting in this paper, we use a smooth version of the inverse-sensitivity mechanism. To this end, we define the $\rho$-smooth inverse-sensitivity of $t$ with respect to norm $\norm{\cdot}$:
\begin{equation}\label{eq:def-sm-inv-sens}
    \sminvmodcont(\Ds;t) = \inf_{s \in \R^d: \norm{s-t} \le \rho}  \invmodcont(\Ds;s).
\end{equation}
The $\rho$-smooth inverse-sensitivity mechanism $\smmech(\cdot;\func)$ then has the following density given input dataset $\Ds$:
\begin{equation}
\label{eq:sm-inv-mech}
    \pi_{\Ds}(t) = \frac{ e^{-\invmodcont^\rho(\Ds;t) \diffp/2}  }{ \int_{s \in \R^d}  e^{-\invmodcont^\rho(\Ds;s) \diffp/2}  ds}
\end{equation}

For our setting of interest where $\range = \{ v \in \R^d: \norm{v} \le R\}$, we have the following upper bound on the error of $\smmech$. Its proof follows similar ideas as in~\citep{AsiD20a, AsiD20b}
.
\begin{theorem}[Continuous functions]
  \label{thm:ub-cont}
  Let $\func: \domain^n \to \range$ where $\range = \{ v \in \R^d: \norm{v} \le R\}$.
  Then for any $\Ds \in \domain^n$, and $\beta >0$, with probability at least $1 - \beta$, the $\rho$-smooth inverse-sensitivity mechanism with norm $\norm{\cdot}$ has error
\ifshort
\else
   \begin{equation*}
   \norm{\smmech(\Ds;\func)-\func(\Ds)}
   \le \modcont_\func\left(\Ds; \frac{2\left(d\log\left(\frac{R}{\rho}+1\right)+\log\frac{1}{\beta}\right)}{\diffp}
    \right) + \rho.
  \end{equation*}
\fi
\end{theorem}
\iflong
\begin{proof}
    We define the good set of outputs $A = \{ t \in \R^d: \sminvmodcont(\Ds;t) \le K\}$. By the definition of $\rho$-smooth inverse-sensitivity, for any $t\in A$, there exists $s \in \R^d$ with $\invmodcont(\Ds;s) \le K$ and $\norm{s-t} \le \rho$. We will show that $\Pr[\smmech(\Ds) \notin A] \le \beta$ for sufficiently large $K$. This implies the desired upper bound as we have that for $t \in A$
    \begin{align*}
    \norm{t - \func(\Ds)}
         & = \norm{t - s + s - \func(\Ds)} \\
         & \le \norm{t - s} + \norm{s  - \func(\Ds)} \\
         & \le \rho + \modcont_\func(\Ds;K).
    \end{align*}
    Now we upper bound $\Pr[\smmech(\Ds;\func) \notin A]$.  First, note that $\sminvmodcont(\Ds;s) = 0$ for $s$ such that $\norm{s-\func(\Ds)} \le \rho$. This implies that for any $t$ such that $\sminvmodcont(\Ds;t) \ge K$, the density is upper bounded by
    \begin{equation*}
    \pi_{\Ds}(t) \le  \frac{e^{-K\diffp/2}}{\int_{s : \norm{s-\func(\Ds)} \le \rho} ds}.
    \end{equation*}
    Overall, this implies that 
    \begin{align*}
    \Pr[\smmech(\Ds;\func) \notin A] 
        & \le e^{-K\diffp/2} \frac{\int_{s: \norm{s} \le R + \rho}  ds}{\int_{s : \norm{s-\func(\Ds)} \le \rho} ds} \\
        & \le e^{-K\diffp/2} (R/\rho + 1)^d.
    \end{align*}
    
    Setting $K\geq \frac{2d \log(R/\rho+1)+2\log(1/\beta)}{\diffp}$, 
    we get that $\Pr(\smmech(\Ds;\func) \notin A)  \le \beta$.
\end{proof}
\fi

\section{Transformations between Robust and Private Algorithms}\label{sec:transformations}
In this section, we provide transformations between differentially private and robust algorithms. We begin in~\Cref{sec:rob-to-priv} with our main result: a general transformation from robust algorithms to private algorithms with roughly the same error for a specified number of corruptions. In~\Cref{sec:priv-to-rob}, we consider the other direction, and show for completeness that any private algorithm is inherently robust as well, which was already observed since~\citep{DworkL09}.

\subsection{Robust to Private}
\label{sec:rob-to-priv}
Our first result shows how to transform a determinitstic robust algorithm into a private algorithm with roughly the same error. The main idea is to apply the $\rho$-smooth inverse-sensitivity mechanism~\citep{AsiD20a} with the input function $\func$ being the robust algorithm itself. 
\begin{theorem}[Robust-to-private]
\label{thm:high-dim-rob-to-priv}
Let $\Ds = (S_1,\dots,S_n)$ where $S_i \simiid P$ such that $\mu(P) \in \R^d.$
Let $\diffp, \beta \in (0,1)$. 
Let $\A_\mathrm{rob} : (\R^d)^n \to \{ t \in \R^d: \norm{t} \le R\}$ be a deterministic $(\tau,\beta,\alpha)$-robust algorithm. Let $\alpha_0\leq \alpha$ and
\begin{equation}\label{eq:fraction}
    \tau\opt= \frac{2\left(d\log\left(\frac{R}{\alpha_0}+1\right)+\log\frac{1}{\beta}\right)}{n\diffp}.
\end{equation}
If $\tau\ge \tau\opt$,     
then $\smmech(\Ds;\A_\mathrm{rob})$ with $\rho = \alpha_0$ is $\diffp$-DP and, with probability at least $1-2\beta$, has error
$
    \norm{\smmech(\Ds;\A_\mathrm{rob}) - \mu} \le 4\alpha. 
$
In particular, this implies that for 
$\tau \geq \frac{2\left(d\log\left(\frac{R}{\alpha_0}+1\right)+\log\frac{2}{\beta}\right)}{n\diffp}$,
\iflong
    \begin{equation*}
        \accp(\diffp,\beta) \le 4\accr(\tau,\beta/2).
    \end{equation*}
\else
    $\accp(\diffp,\beta) \le 4\accr(\tau,\beta/2).$
\fi
\end{theorem}

\begin{proof}
First note that the privacy guarantee is immediate from the guarantees of the exponential mechanism (\Cref{def:expmech}) and the fact that the sensitivity of the $\rho$-smooth path-length function in~\Cref{eq:def-sm-inv-sens} is $1$. Now we prove utility. 
Let $K = n\tau\opt$. 
The error of $\smmech(\Ds,\A_\mathrm{rob})$ is then bounded as follows:
\iflong
    \begin{align*}
        & \norm{\smmech(\Ds;\A_\mathrm{rob})- \mu} \\
        & \le \norm{\A_\mathrm{rob}(\Ds)-\mu}+ \norm{\smmech(\Ds;\A_\mathrm{rob})- \A_\mathrm{rob}(\Ds)} \tag{by triangle inequality}\\
        & \le \norm{\A_\mathrm{rob}(\Ds)-\mu}+\sup_{\Dsc: \dham(\Dsc,\Ds)\leq K}\norm{\A_{\mathrm{rob}}(\Dsc)-\A_{\mathrm{rob}}(\Ds)} + \alpha_0 \tag{w.p. $1-\beta$ by~\Cref{thm:ub-cont}}\\
        & \le 2 \norm{\A_\mathrm{rob}(\Ds)-\mu} + \sup_{\Dsc: \dham(\Dsc,\Ds)\leq K}\norm{\A_\mathrm{rob}(\Dsc)-\mu} +\alpha_0. \tag{by triangle inequality}
    \end{align*}
\else
    \begin{align*}
        & \norm{\smmech(\Ds;\A_\mathrm{rob})- \mu} \\
        & \le \norm{\A_\mathrm{rob}(\Ds)-\mu}+ \norm{\smmech(\Ds;\A_\mathrm{rob})- \A_\mathrm{rob}(\Ds)}\\
        & \le \norm{\A_\mathrm{rob}(\Ds)-\mu}+\modcont_{\A_\mathrm{rob}}\left(\Ds;K\right) + \alpha_0 \tag{w.p. $1-\beta$ by~\Cref{thm:ub-cont}}\\
        & \le 2 \norm{\A_\mathrm{rob}(\Ds)-\mu} + \norm{\A_\mathrm{rob}(\Dsc)-\mu} +\alpha_0,
    \end{align*}
    for $\Dsc=\argmax_{\Dsc: \dham(\Dsc,\Ds)\leq K}\norm{\A_\mathrm{rob}(\Ds)-\A_\mathrm{rob}(\Dsc)}$.
\fi
Recall that, by assumption, $\A_\mathrm{rob}$ is $(\tau,\beta,\alpha)$-robust for $\tau\geq K/n$. 
Thus, with probability $1-\beta$, $\norm{\A_\mathrm{rob}(\Dsc)-\mu}\leq \alpha$ for any $\tau$-corrupted dataset $\Dsc$. 
By union bound, we have that with probability $1-2\beta$, $\norm{\smmech(\Ds;\A_\mathrm{rob})-\mu}\leq 3\alpha+\alpha_0\leq 4\alpha$. This completes the proof of the theorem.
\end{proof}
The parameter $\alpha_0$ determines the smallest fraction of corruptions $\tau\opt$, which in turn determines the smallest $\alpha$ so that $\A_{\mathrm{rob}}$ is $(\tau\opt,\beta,\alpha)$-robust. 
\iflong
     There exists a choice of parameter $\alpha_0$ which balances the two sources of error so that $\alpha=\Theta(\alpha_0)$. 
     However, since $\alpha_0$ only affects logarithmic factors, for simplicity, in our applications we prefer to set $\alpha_0$ to be smaller 
     -- in particular, we set it to 
\else
    A simple choice for $\alpha_0$ is 
\fi
the minimax error for estimating the statistic $\mu(P)$, without adversarial corruptions or privacy 
\iflong constraints, i.e., $\alpha_0=\accr(0,\beta)$.
\else constraints.\fi

\begin{remark} We can extend this transformation to hold for randomized robust algorithms, by first converting $\A_\mathrm{rob}$ into a deterministic algorithm, albeit doubling the error and failure probability, as shown in~\Cref{thm:rand-to-det},~\Cref{app:rand-to-det}. \end{remark}

\subsection{Private to Robust}
\label{sec:priv-to-rob}
In this section, we state the folklore fact that any $\eps$-differentially private algorithm is also $\tau \approx\frac{1}{n \diffp}$-robust with the same accuracy. This follows directly from the definition of differential privacy which states that changing $1/\diffp$ users does not change the output distribution by much (by \emph{group privacy}) and was observed in~\citep{DworkL09}. 
\begin{theorem}[Private-to-robust]\label{thm:priv-to-rob}
Let $\Ds = (S_1,\dots,S_n)$ where $S_i \simiid P$. Let $\diffp, \beta\in(0,1)$.
Let $\A_\mathrm{priv}$ be an $(\diffp,\beta,\alpha)$-private algorithm for estimating the statistic $\mu$. Let $\gamma\in(0,1)$ and $\tau=\frac{\log(1/\gamma)}{n\diffp}$. Then $\A_\mathrm{priv}$ is $\left(\tau,\beta/\gamma,\alpha\right)$-robust. In particular, $\accr(\tau,\beta/\gamma) \le \accp(\diffp,\beta)$, which is equivalent to
\iflong
  \begin{equation*}
    \accr(\tau,\beta) \le \accp\left(\frac{\log(1/\gamma)}{n\tau},\gamma \beta \right).
  \end{equation*}
\else
    $\accr(\tau,\beta) \le \accp(\log(1/\gamma)/n\tau,\gamma \beta ).$
\fi
\end{theorem}
\begin{proof}
Let $W = \{t\in\R^d : \norm{t-\mu} >\alpha \}$ be the set of bad outputs for the distribution $P$.
The accuracy guarantee of $\A_\mathrm{priv}$ implies that
\iflong
    \begin{equation*}
        \Pr[\A_\mathrm{priv}(\Ds) \in W] \le \beta.
    \end{equation*}
\else
    $\Pr[\A_\mathrm{priv}(\Ds) \in W] \le \beta.$ 
\fi
Now assume $\Dsc$ is a $\tau n$-corrupted version of $\Ds$ where $\tau = \log(1/\gamma)/(n\diffp)$, that is, $\dham(\Ds,\Dsc) \le \log(1/\gamma)/\diffp$. The definition of differential privacy now implies that
\begin{align*}
    \Pr[\norm{\A_\mathrm{priv}(\Dsc)-\mu} > \alpha]
    & = \Pr[\A_\mathrm{priv}(\Dsc) \in W] \\
    & \le e^{\dham(\Ds,\Dsc)\diffp} \Pr[\A_\mathrm{priv}(\Ds) \in W] \\
    & \le e^{\diffp\tau n} \beta \le \beta/\gamma,
\end{align*}
for $\tau\leq \frac{\log(1/\gamma)}{n\diffp}$.
Thus $\A_\mathrm{priv}$ is $\tau$-robust for $\tau = \frac{\log(1/\gamma)}{n\diffp}$ with accuracy $\alpha$ and failure probability $\beta/\gamma$.
\end{proof}
We note that for $\gamma=1/e$, $\accr(\tau,\beta) \le \accp\left(1/n\tau,\beta/e \right)$, that is, the minimax error of any $\tau$-robust algorithm with failure probability $\beta$ is bounded by the minimax error of a $\diffp$-DP  algorithm, for $\diffp=1/n\tau$, with the same failure probability up to constant factors.

As private algorithms are often randomized, this transformation would result in a randomized robust algorithm. As in the previous section, we can convert it into a deterministic one via~\Cref{thm:rand-to-det} in~\Cref{app:rand-to-det}.

\section{Implications of our transformations}\label{sec:implications}
\iflong
Building on our transformations in~\Cref{sec:transformations} between private and robust algorithms, in this section we show two results. First, we show that $\diffp$-DP statistical estimation is equivalent to $\tau$-robust estimation for low-dimensional problems when $\tau \approx \log(n)/n\diffp$, that is, both problems have the same minimax rates (with slightly different failure probabilities). Then, we prove that for any low-dimensional statistical query, our transformation from robust to private estimators can be instantiated to obtain optimal utility up to constants.
\fi

\subsection{Equivalence between Private and Robust Estimation}
\label{sec:eqv}
In this section, we show that the (high-probability) minimax rates for $\diffp$-DP and $\tau$-robustness are on the same order when the problem is low-dimensional and $\tau = \log(n)/n\diffp$. The following corollary states the result. For simplicity, we assume that the dimension $d=1$ and the range $R=1$. 

\begin{corollary}[Equivalence]
\label{cor:eqv}
    Let $\P$ be a family of distributions and $P\in\P$. Let $\mu$ be a $1$-dimensional statistic where $|\mu(P)| \le 1$ such that $\accr(\tau,\beta)$ is a continuous function of $\beta$ for all $\tau \le 1/2$.
    Let $n > 1$, $\diffp = \omega(\log(n)/n)$, and $\tau = \Theta(\log(n)/n\diffp)$. 
    Suppose there exists a constant $c$ such that the error $\alpha_{\mathrm{rob}}(0,\beta)\geq \frac{1}{n^c}$ for any $\beta \leq 1/4$. Then there are constants $c_1 \ge c_2 >0$ such that for $\beta_p = 1/n^{c_1}$ and $\beta_r = 1/n^{c_2}$, 
    \begin{equation*}
        \accp(\diffp,\beta_p) = \Theta \left( \accr(\tau,\beta_r) \right).
    \end{equation*}
\end{corollary}
\begin{proof}
First, we observe that $\alpha_0 = \accr(0,\beta_p)\leq \accr(\tau,\beta_p)$ by the monotonicity of $\accr$ and $\alpha_0\ge \frac{1}{n^c}$ since $\beta_p=\frac{1}{n^{c_1}}\leq 1/2$. By~\Cref{thm:high-dim-rob-to-priv}\footnote{If the robust algorithm achieving the minimax error $\accr$ is randomized, we can transform it into a deterministic one as~\Cref{thm:high-dim-rob-to-priv} requires, via~\Cref{thm:rand-to-det}, by losing only constant factors.}, we have that, for $\tau_1 = \frac{2(c+c_1)\log(2n)}{n\diffp} \geq \frac{2\log(1/\alpha_0+1)+\log(2/\beta_p)}{n\diffp}$, 
\begin{align}\label{eq:eqv}
    \accp(\diffp,\beta_p) 
       & \le 4\accr(\tau_1,\beta_p/2).
\end{align}
Setting $\gamma = 1/(2n)^{2(c+c_1)}$, we have that $\tau_1=\frac{\log(1/\gamma)}{n\diffp}$, and by~\Cref{thm:priv-to-rob} we have that 
\begin{equation*}
    \accr(\tau_1,(2n)^{2(c+c_1)}\beta) \le \accp(\diffp, \beta ).
\end{equation*}
Note that if $\accp(\diffp, \beta_p ) \ge \accr(\tau_1, \beta_p/2)$ then the claim follows from~\Cref{eq:eqv}, using $\beta_r = \beta_p/2$.
Otherwise we have that 
\begin{equation*}
    \accr(\tau_1,(2n)^{2(c+c_1)}\beta_p) \le \accp(\diffp, \beta_p ) \le \accr(\tau_1, \beta_p/2)
\end{equation*} 
As $\accr(\tau_1, \beta)$ is a continuous function of $\beta$, there is $\beta_r \in [\beta_p/2, (2n)^{2(c+c_1)}\beta_p]\subset [\frac{1}{n^{2c_1}}, \frac{1}{n^{2c+3c_1}}]$ such that $\accr(\tau, \beta_r) = \accp(\diffp, \beta_p )$.
\end{proof}

Note that in most settings, $\accr(\tau,\beta_r)$ has the same order when $\beta_r \in [\beta_p,\mathrm{poly}(n)\cdot\beta_p]$ as it depends on $\log(1/\beta)$ (see for example~\Cref{sec:applications}). \iflong Moreover, we can amplify the success probability of the algorithm by repetition, increasing its error by a multiplicative $\mathsf{poly}(\log(\frac{1}{\beta}))$ factor. \ha{i think we need poly here..} \fi

\subsection{Optimality of Black-Box Transformation}
\label{sec:opt}

An immediate corollary of the previous transformations is that---for some choice of robust algorithm $\A_{\mathrm{rob}}$---our robust-to-private transformation achieves the minimax optimal rate among the family of private algorithms, for low-dimensional statistics. 
\begin{corollary}[Optimality]\label{cor:optimality}
    Let $\P$ be a family of distributions and $P\in\P$. Let $\mu$ be a $1$-dimensional statistic where $|\mu(P)| \le 1$. Let $\accp(\diffp,\beta)$ be the minimax error of any $\diffp$-DP algorithm with failure probability $\beta\leq 1/4$ that estimates statistic $\mu(P)$. Let $n>1$. Suppose that there exists a constant $c$ such that the non-private error $\accp(\infty, \beta)\geq \frac{1}{n^c}$ for any $\beta\leq 1/2$. 
    Then there are constants $c_1 \ge c_2 >0$ such that $\beta_p = 1/n^{c_1}$ and $\beta_p' = 1/n^{c_2}$, robust algorithm $\A_\mathrm{rob}$, and a choice of $\rho$, such that Algorithm  $\smmech(\cdot;\A_\mathrm{rob})$ with privacy parameter $\diffp$ achieves the optimal error $O(\accp(\diffp,\beta_p))$ with probability $1-\beta_p'$. 
\end{corollary}
\iflong
\begin{proof}
Let $\accp(\diffp, \beta)$ be the minimax error for estimating $\mu$ under family $\P$ for any $\beta\leq 1/4$.
By~\Cref{thm:priv-to-rob} and~\Cref{thm:rand-to-det}, there exists a deterministic $\tau$-robust algorithm with $\tau=\log(1/\gamma)/n\diffp$, with accuracy $\alpha_1=2\accp(\diffp, \beta)$ and failure probability $\beta_1=2\beta/\gamma$. 
Via the transformation of~\Cref{thm:high-dim-rob-to-priv}, choosing $\rho=\frac{1}{n^c}\leq \accp(\diffp, \beta)$, we can construct an $\diffp$-DP algorithm with failure probability $2\beta_1=4\beta/\gamma$ and accuracy $4\alpha_1=8\accp(\diffp, \beta)$, if $\tau=\frac{\log(1/\gamma)}{n\diffp}\geq \frac{2\log(n^c+1)+2\log(\gamma/(4\beta))}{n\diffp}$. 
Setting $\gamma=(4\beta/(2n)^c)^{2/3}$ 
satisfies the requirement. 
Thus, the $\diffp$-DP algorithm constructed via the transformation in~\Cref{thm:high-dim-rob-to-priv} has error at most $8\accp(\diffp, \beta)$ with failure probability at most $\beta_p'=4\beta/\gamma=(4\beta)^{1/3}\cdot (2n)^{2c/3}$. There exist constants $c_1\geq c_2>0$ such that $\beta=\beta_p=\frac{1}{n^{c_1}}$ and $\beta_p'=\frac{1}{n^{c_2}}$.\end{proof} 
\fi


\iflong \section{Improved Transformation for Approximate DP}\else \section{A Transformation for Approximate DP}\fi\label{sec:approxDP}

In this section, we propose a different transformation for \ed-DP that avoids the necessary dependence on diameter for pure $\diffp$-DP
\iflong. Our transformation uses a truncated version of the inverse-sensitivity mechanism which only outputs values with bounded inverse sensitivity. This mechanism is not differentially private for all inputs, therefore, in order to guarantee privacy, we use a private test to verify that the input is well-behaved before running the truncated inverse-sensitivity mechanism. 
    
This approach can be viewed as a special case of the restricted exponential mechanism of~\citet{BrownGSUZ21} (or the even more general HPTR framework~\citep{LiuKO22}), which in turn has been inspired by the propose-test-release (PTR) framework~\citep{DworkL09}. However, we choose a simplified algorithm and presentation, which is tailored to our case, where we have the smooth inverse sensitivity as our cost function. 
    
\subsection{Truncated Inverse-Sensitivity Mechanism}
We develop a truncated version of the inverse-sensitivity mechanism which is \ed-DP. 
This mechanism uses a truncated version of the inverse sensitivity as follows: given a function $f: \domain^n \to \R^d$ and threshold $K$,
\iflong 
    \begin{equation*}
      \label{eqn:inverse-trunc}
      \invmodcont^{\mathsf{trunc}}_\func(\Ds; t) \defeq
      \begin{cases}
        \invmodcont^{\rho}_\func(\Ds; t) & \text{if  } \invmodcont^\rho_\func(\Ds; t) \le K \\
        \infty & \text{otherwise}. 
      \end{cases}
    \end{equation*}
\else
    \begin{equation*}
      \label{eqn:inverse-trunc-app}
      \invmodcont^{\mathsf{trunc}}_\func(\Ds; t) \defeq
      \begin{cases}
        \invmodcont^{\rho}_\func(\Ds; t) & \text{if  } \invmodcont^\rho_\func(\Ds; t) \le K \\
        \infty & \text{otherwise}. 
      \end{cases}
    \end{equation*}
\fi
The truncated inverse-sensitivity mechanism $\trmech(\cdot;\func)$ then applies the exponential mechanism using this score function, resulting in the following density given an input dataset $\Ds$:
\begin{equation}
\label{eq:trunc-inv-mech}
    \pi_{\Ds}(t) = \frac{ e^{-\invmodcont^{\mathsf{trunc}}_\func(\Ds;t) \diffp/2}  }{ \int_{s \in \R^d}  e^{-\invmodcont^{\mathsf{trunc}}_\func(\Ds;s) \diffp/2}  ds}
\end{equation}
    
Before proving the guarantees of the truncated inverse-sensitivity mechanism, we need to define $(\diffp,\delta)$-indistinguishable distributions:
\begin{definition}[$(\diffp,\delta)$-indistinguishability]\label{def:indistinguishable}
    Two distributions $P,Q$ over domain $\range$ are $(\diffp,\delta)$-indistinguishable, denoted by $P \approx_{\eps,\delta} Q$, if for any measurable subset $T\subseteq \range$, \[\Pr_{t\sim P}[t\in T] \leq e^\eps \Pr_{t\sim Q}[t\in T] + \delta \quad\text{ and }\quad \Pr_{t\sim Q}[t\in T] \leq e^\eps \Pr_{t\sim P}[t\in T] + \delta.\]
\end{definition}
Note that if $\A(\Ds)\approx_{\diffp,\delta} \A(\Dsc)$ for any neighboring datasets $\Ds,\Dsc$, then $\A$ is $(\diffp,\delta)$-differentially private.
We have the following guarantees for the truncated inverse-sensitivity mechanism.
\begin{proposition}
    \label{prop:trunc-inv}
    Let $n\ge 1$, $\diffp, \delta \in (0,1)$, $B>0$, and $\func: \domain^n \to \R^d$. Let $K \ge \frac{d + \log(1/\delta)}{\diffp}$ and $S_{\mathsf{bad}} = \{\Ds \in \domain^n : \modcont_\func(\Ds;K+1) > B \}$.
    For any $\Ds \notin S_{\mathsf{bad}}$, the truncated inverse-sensitivity mechanism~\eqref{eq:trunc-inv-mech} with $\rho=2B$ has error
    \begin{align*}
    \norm{\trmech(\Ds;\func)-\func(\Ds)} 
    \le 3 B.
\end{align*}
Moreover, for any $\Ds\notin S_{\mathsf{bad}}$ and neighboring dataset $\Dsc$, $\trmech(\Ds;\func)\approx_{\diffp,\delta} \trmech(\Dsc;\func)$.
\end{proposition}
\begin{proof}
The claim about utility follows directly from the definition of the truncated inverse-sensitivity as the probability of returning $t$ such that  $\invmodcont^{\mathsf{trunc}}_\func(\Ds; t) \ge K$ is zero. Now we proceed to prove the privacy claim. Let $\Ds \in S_{\mathsf{bad}}^c$ and $\Ds' \in \domain^n$ be two neighboring datasets and $T \subset \range$. Since $\modcont_\func(\Ds;K+1) \le B $, we have that $\modcont_\func(\Ds';K) \le 2B $. 
Thus, it suffices to show that for any two neighboring datasets $\Ds$ and $\Ds'$ such that $\modcont_\func(\Ds';K) \le 2B$ and $\modcont_\func(\Ds;K) \le 2B$, we have $\Pr[\trmech(\Ds;\func) \in T] \le e^\diffp \Pr[\trmech(\Dsc;\func) \in T] + \delta$. 
Let $T_0 = \{ t \in \range: \invmodcont^\rho_\func(\Ds; t) = K \}$. Now we have
    \begin{align*}
        \Pr[\trmech(\Ds;\func) \in T] 
        & = \Pr[\trmech(\Ds;\func) \in T \setminus T_0]  + \Pr[\trmech(\Ds;\func) \in T \cap T_0] \\
        & \le e^\diffp \Pr[\trmech(\Dsc;\func) \in T \setminus T_0] + \frac{e^{-K\diffp}}{\vol(\ball^d(\rho))} \mathsf{Vol}(T \cap T_0) \\
        & \le e^\diffp \Pr[\trmech(\Dsc;\func) \in T] + \frac{e^{-K\diffp}}{\vol(\ball^d(\rho))} \vol(T \cap T_0),
\end{align*}
where the first inequality follows since for $t \notin T_0$ we have that either $|\invmodcont^{\mathsf{trunc}}_\func(\Ds; t) - \invmodcont^{\mathsf{trunc}}_\func(\Ds'; t)| \le 1$ or $\invmodcont^{\mathsf{trunc}}_\func(\Ds; t) = \infty$. 
Since $\modcont_\func(\Ds;K) \le 2B $, we get $\vol(T \cap T_0) \le \vol(T_0)\leq \vol(\ball^d(2B+\rho))$. For $\rho = 2B$, this implies that $\vol(T\cap T_0)/\vol(\ball^d(\rho))\leq \vol(\ball^d(4B))/\vol(\ball^d(2B))=2^d$. Therefore, for $K \ge \frac{d + \log(1/\delta)}{\diffp}$, we have that $\Pr[\trmech(\Ds;\func) \in T]\leq e^\diffp \Pr[\trmech(\Ds';\func) \in T] + \delta$. By symmetry, we can use the same argument, to show that $\Pr[\trmech(\Ds';\func)\in T]\le e^{\diffp}\Pr[\trmech(\Ds;\func)\in T]+\delta$. Thus, overall we show that $\trmech(\Ds;\func)\approx_{\diffp,\delta}\trmech(\Dsc;\func)$ for all $\Ds\notin S_{\mathsf{bad}}$ and neighboring dataset $\Dsc: \dham(\Ds,\Dsc)\leq 1$.  
\end{proof}
While it may seem that the truncated inverse sensitivity provides the desired transformation for \ed-DP, note that it requires the condition $\modcont_\func(\Ds;K+1) \le B $ to hold for all inputs $\Ds \in \domain^n$. However, robust algorithms only guarantee boundedness of  $\modcont_\func(\Ds;K+1) $ for $\Ds \simiid P^n$. 
To this end, in the next section we show how to use propose-test-release (PTR) in order to overcome this barrier.
\else, by using the following truncated version of the inverse-sensitivity:
    \begin{equation*}
        \label{eqn:inverse-trunc}
        \invmodcont^{\mathsf{trunc}}_\func(\Ds; t) \defeq
        \begin{cases}
            \invmodcont^{\rho}_\func(\Ds; t) & \text{if  } \invmodcont^\rho_\func(\Ds; t) \le K \\
            \infty & \text{otherwise}. 
        \end{cases}
    \end{equation*}
\fi    

\iflong \subsection{PTR-based Transformation}
Building on the truncated inverse-sensitivity mechanism, in this section we use propose-test-release (PTR) to design a transformation from robust algorithms into approximate \ed-DP algorithms where the error does not depend on the diameter. 
    
An equivalent approach would be to use the restricted exponential mechanism from~\citep{BrownGSUZ21} with the smooth inverse-sensitivity $\sminvmodcont_\func(\Ds;t)$ as its cost function. The main idea of this approach is to perform a private test to check if the input $\Ds$ is far from ``unsafe'', before running the exponential mechanism restricted to points $t$ with $\sminvmodcont_\func(\Ds;t)\leq K$. The set ``unsafe'' consists of datasets on which running the restricted exponential mechanism would not produce $(\diffp,\delta)$-indistinguishable outputs. However, our specific score function allows us to simplify the ``unsafe'' set, and this is the algorithm we present in this section. \fi
\Cref{alg:transf-appr} uses a private test to check whether $\Ds$ is far from the set $S_{\mathsf{bad}}=\{\Ds\in\domain^n: \modcont_\func(\Ds;K+1) > B\}$. If it is not, then it fails. Crucially, if $\smash{\Ds \simiid P^n}$ then the robust algorithm guarantees that $\modcont_\func(\Ds';K+1) \le B$ for $\Ds'$ in a neighborhood of $\Ds$, allowing the test to succeed in this case.

\begin{algorithm}[t]
	\caption{Robust-to-Private $\left(\ed\text{-DP} \right)$}
	\label{alg:transf-appr}
	\begin{algorithmic}[1]
		\REQUIRE $\Ds = (S_1,\dots,S_n)$, $(\tau,\beta,\alpha)$-robust algorithm $\A_\text{rob}$, local modulus bound $B$
	\STATE Let $K = n\tau/2-1$
        \STATE Let $S_{\mathsf{bad}} = \{\Ds \in \domain^n : \modcont_\func(\Ds;K+1) > B \}$
        \STATE Calculate $d = \min_{\Dsc\in S_\mathsf{bad}}\dham(\Ds, \Dsc)$
        \STATE Set $\hat d = d + \zeta$ where $\zeta \sim \lap(2/\diffp)$
        \IF{$\hat d > 2\log(1/\min(\delta,\beta))/\diffp$} 
            \STATE Sample $t$ from the truncated inverse-sensitivity mechanism \iflong \eqref{eq:trunc-inv-mech} \fi with threshold $K$, privacy parameter $\diffp/2$, smoothness parameter $\rho=2B$, and return $t$.
        \ELSE
            \STATE Return $\perp$
        \ENDIF
	\end{algorithmic}
\end{algorithm}
\begin{theorem}[Robust-to-private, approximate DP]
\label{thm:rob-to-priv-appr}
    Let $\Ds = (S_1,\dots,S_n)$ where $S_i \simiid P$ such that $\mu(P) \in \R^d$. Let $\diffp, \delta, \beta \in (0,1)$.
    Let $\A_\text{rob} : (\R^d)^n \to \R^d$ be a deterministic $(\tau,\beta,\alpha)$-robust algorithm for the statistic $\mu$. If $\tau \ge \frac{8 (d + \log(1/\min\{\delta,\beta\}))}{n\diffp}$ then~\Cref{alg:transf-appr} with $B = 2 \alpha$ and $\rho = 2B$ is \ed-DP and, with probability at least $1-2\beta$ returns $\hat \mu$ such that
    $
        \norm{\hat \mu - \mu} \le 7\alpha.
    $
\end{theorem}
\iflong 
\begin{proof}    
    We start by proving the privacy guarantees of~\Cref{alg:transf-appr}. Note that the Laplace mechanism~\cite{DworkMNS06} implies that $\hat d$ is $\diffp/2$-DP as the function $d$ has sensitivity $1$. By assumption on $\tau$, $K=n\tau/2-1\geq \frac{2(d+\log(2/\delta))}{\diffp}$. Thus, by~\Cref{prop:trunc-inv}, for input dataset $\Ds$, if $\modcont_\func(\Ds;K+1) \le B $ then the truncated inverse-sensitivity is $(\diffp/2,\delta/2)$-DP. On the other hand, if $\modcont_\func(\Ds;K+1) > B $ then $d = 0$ and therefore $\hat d \leq 2 \log(1/\delta)/\diffp$ with probability $1 - \delta/2$ and the algorithm returns $\perp$. Overall, by composition,~\Cref{alg:transf-appr} is \ed-DP.

    We now prove the accuracy guarantee. Let $\Ds \simiid P^n$. By the guarantee of the robust algorithm, with probability $1-\beta$, for all $\Ds'$ such that $\dham(\Ds,\Ds') \le \tau n$, we get that $\norm{\A_\text{rob}(\Ds') - \mu(P)} \le \alpha$. Therefore, for any $\Ds'_1$ such that $\dham(\Ds,\Ds'_1) \le \tau n/2$, we have
    \begin{align*}
    \modcont_\func(\Ds'_1;n \tau/2)
     & = \sup_{\Ds'_2: \dham(\Ds'_1,\Ds'_2) \le n\tau/2} \norm{\A_\text{rob}(\Ds'_1) - \A_\text{rob}(\Ds'_2)} \\
     & \le \sup_{\Ds'_2: \dham(\Ds'_1,\Ds'_2) \le n\tau/2} \left(\norm{\A_\text{rob}(\Ds'_1) - \mu(P)} + \norm{\mu(P)  - \A_\text{rob}(\Ds'_2)}\right) \\
     &\le \alpha + \sup_{\Ds'_2: \dham(\Ds,\Ds'_2) \le n\tau}\norm{\mu(P)  - \A_\text{rob}(\Ds'_2)} \\
     & \le 2 \alpha.
    \end{align*}
    
    Since $B=2\alpha$, we have that with probability $1-\beta$, $d = > n\tau/2$ and in particular $\Ds\notin S_\mathsf{bad}$. By the concentration guarantees of the Laplace distribution, we have that $\hat d > n\tau/2 - 2\log(1/\beta)/\diffp$ with probability at least $1 - \beta$, and thus $\hat d >\frac{2\log(1/\min\{\beta,\delta\})}{\diffp}$, which implies that the algorithm will run the truncated inverse-sensitivity mechanism. \Cref{prop:trunc-inv} now implies that the latter will return $\hat \mu$ such that $\norm{\hat \mu - \A_\text{rob}(\Ds)} \le 3B$. Moreover, $\norm{ \A_\text{rob}(\Ds) - \mu} \le \alpha$. Overall we get that with probability $1-2\beta$,
    \begin{equation*}
    \norm{\hat \mu-\mu}
        \le \norm{\hat \mu - \A_\text{rob}(\Ds)}  + \norm{ \A_\text{rob}(\Ds) - \mu} \le 3B+\alpha=7\alpha.
    \end{equation*}
    This completes the proof of the theorem.
\end{proof} \fi

\section{Applications for Pure DP}
\label{sec:applications}
\iflong
    In this section we apply our main transformation in~\Cref{thm:high-dim-rob-to-priv} to fundamental tasks in private statistics to demonstrate that for all these tasks near-optimal error can be achieved by instantiating our black-box reduction with a robust estimator for the same task. 
In~\Cref{sec:gauss-mean} and~\Cref{sec:gauss-covariance} we show that we can retrieve known optimal results for mean and covariance estimation of Gaussian distributions up to logarithmic factors. In~\Cref{sec:gauss-linear} and~\Cref{sec:gauss-pca}, we show that our transformation gives the first algorithms with optimal error for linear regression (including the sparse case) and PCA for Gaussian distributions. Our results for PCA hold for subgaussian distributions more generally.

For the majority of this section we will use the more general transformation, proven in~\Cref{app:gen-loss}\ifshort
\else, and stated in~\Cref{thm:high-dim-rob-to-priv-loss} below. 
    The task of spherical Gaussian mean estimation is an exception, where we use the simpler~\Cref{thm:high-dim-rob-to-priv}, where the loss function is the $\ell_2$ norm.
    \begin{theorem}[Robust-to-private, general loss]
    \label{thm:high-dim-rob-to-priv-loss}
      Let $\Ds = (S_1,\dots,S_n)$ where $S_i \simiid P$ such that $\mu(P) \in \R^d.$
      Let $\diffp, \beta \in (0,1)$. Let $L:(\R^d)^2\to \R$ be a loss function which satisfies the triangle inequality.
      Let $\A_\mathrm{rob} : (\R^d)^n \to \{ t \in \R^d: \norm{t} \le R\}$ be a (deterministic) $(\tau,\beta,\alpha)$-robust algorithm with respect to $L$. 
      Let $\alpha_0\leq \alpha$. Suppose $n$ is such that the smallest value $\tau$ satisfying~\Cref{eq:fraction-loss} is at most a known small constant $\tau_0$. Suppose for all $u,v\in\ball(R+\alpha_0)$, $ L(u,v)\leq c_L\norm{u-v}$ for some constant $c_L$. 
      If 
      \begin{equation}\label{eq:fraction-loss}
      \tau \geq \frac{2\left(
        d\log\left(\frac{R}{\alpha_0}+1\right)+\log\frac{1}{\beta}\right)}{n\diffp},
        \end{equation}
        then Algorithm $\smmech(\Ds;\A_\mathrm{rob})$ with $\rho = \alpha_0$ in norm $\norm{\cdot}$ is $\diffp$-DP and, with probability at least $1-2\beta$, has error
      \begin{equation*}
          L\left(\smmech(\Ds;\A_\mathrm{rob}),\mu\right) \le (3+c_L)\alpha=O(\alpha).
      \end{equation*}
      \end{theorem}
\fi

The general strategy we follow in our applications is simple. We choose a known robust algorithm $\A$ for the statistic $\mu\in\ball(R)$ we want to estimate. 
Informally, let us denote its accuracy by $\alpha(\tau)$, as it will be a function of the fraction of corruptions in the dataset $\tau$ (among other parameters). 
Applying our robust-to-private transformation from~\Cref{thm:high-dim-rob-to-priv}, we retrieve an $\diffp$-DP algorithm with accuracy roughly $\alpha(\tau\opt)$ for $\tau\opt\approx \frac{d\log(R'/\alpha_0)+\log(1/\beta)}{n\diffp}$. 
More precisely, we let $\A_\mathrm{rob}$ be the algorithm that runs $\A$ and then projects its output on $\ball(R')$, where $R'$ is such that, with high probability, the projection will have no effect and will maintain the accuracy guarantees of $\A$. 
Let $\alpha_0$ be the error rate for learning statistic $\mu$ without privacy constraints or corruptions, which is always smaller than $\alpha(\tau)$. 
We run the $\rho$-smooth-inverse-sensitivity mechanism instantiating it with the projected robust algorithm $\A_\mathrm{rob}$ and with smoothness parameter $\rho=\alpha_0$. 
In most applications, $\alpha(\tau)=\tilde{O}(\tau)$ so the error we incur on top of the non-private error $\alpha_0$ is $\tilde{O}(d/n\diffp)$.

\iflong\Cref{thm:high-dim-rob-to-priv-loss} \else \Cref{thm:high-dim-rob-to-priv-loss-app}\fi extends~\Cref{thm:high-dim-rob-to-priv} allowing us to measure the error of the algorithm with respect to a loss function $L$ that may depend on unknown parameters and thus can not be computed directly. As long as this loss satisfies the triangle inequality, and any error we incur due to the smoothness $\rho$ in norm $\norm{\cdot}$ upper-bounds the error in $L$ up to constants, the statement of our main theorem still holds. 

For useful linear algebra facts and definitions, see~\Cref{app:facts}.

\subsection{Mean Estimation}\label{sec:gauss-mean}
\subsubsection{Known Covariance}
We start with the task of estimating the mean $\mu$ of a $d$-dimensional Gaussian distribution with known covariance $\Sigma$. By applying $\Sigma^{-1/2}$ to all the points, this case can be reduced to spherical Gaussian mean estimation, where we can assume $\Sigma=\id$. We also assume that we know \emph{a priori} a bound $R$ such that $\ltwo{\mu}\leq R$.\footnote{Knowledge of $R$ is necessary for mean estimation under pure DP~\citep{HardtT10, BeimelBKN14, BunKSW19}.} 
~\Cref{cor:sph-gauss-mean} states that via our transformation, we can retrieve the optimal error for Gaussian mean estimation with known covariance under pure DP, matching optimal bounds~\citep{BunKSW19, LiuKK021}.
\begin{corollary}[Spherical Gaussian mean]\label{cor:sph-gauss-mean}
Let $\Ds = (S_1,\dots,S_n)$ where $S_i \simiid \normal(\mu,\id)$ such that $\mu \in \ball(R)$. Let $\diffp, \beta\in (0,1)$ and $C \ge 1$ a known constant. Suppose $n$ is such that $\alpha\leq 1$ in~\Cref{eq:sph-gauss-mean}.
There exists an $\diffp$-DP algorithm $\mc{M}$ such that, with probability at least $1-\beta$, has error $\ltwo{\mc{M}(\Ds) - \mu} \le \alpha$ for
\begin{equation}\label{eq:sph-gauss-mean}
    \alpha=C\cdot\left(\sqrt{\frac{d+\log(\frac{1}{\beta})}{n}} + \frac{d\log\left(\frac{Rn}{d}\right)+\log(\frac{1}{\beta})}{n\diffp}\right).
\end{equation}
\end{corollary}
Since we are not concerned with computational efficiency, we will use the \emph{Tukey median} as the robust Gaussian mean estimation algorithm for our transformation. The Tukey depth~\cite{Tukey60} of a point $t$ with respect to a distribution 
$P$ is defined by \begin{align*}
    T_P(t) \defeq \inf_{v\in\R^d} \Pr_{S\sim P}[\langle S,v\rangle\ge \langle t,v\rangle].
\end{align*}
We denote by $T_{\Ds}(t)\defeq \frac{1}{n}\min_v \sum_{i\in [n]}[\langle S_i,v\rangle\ge \langle t,v\rangle]$ the (normalized) Tukey depth of $t$ with respect to dataset $\Ds$. 
The Tukey median with respect to any dataset $\Ds$ is then $t_m(\Ds)=\argmax_{t\in \R^d} T_{\Ds}(t)$. 
Let $\Pi_{\C}(t)=\argmin_{v\in\C} \ltwo{v-t}$ be the euclidean projection of a point $t$ to convex set $\C$.
The next proposition states the robustness guarantees of (projected) Tukey median, which have been long-established (for a complete proof see e.g.~\citep{Li19} or the more general~\Cref{prop:accr-tukey-app} in~\Cref{app:facts}).
\begin{proposition}\label{prop:sph-accr-tukey}
Let $\Ds = (S_1,\dots,S_n)$ where $S_i \simiid \normal(\mu,\id)$ such that $\mu \in \ball(R)$. Let $\beta\in(0,1)$, $\tau\leq 0.05$, and $\alpha_0=C_0\cdot\sqrt{(d+\log(1/\beta))/n}$ for a known constant $C_0\ge 1$. Suppose $n$ is such that $\alpha_0\leq 0.05$. Let $\alpha=7(\alpha_0+\tau)\leq 1$.
The projected Tukey median algorithm 
$\A_{\mathrm{rob}}(\Ds)=\Pi_{\ball\left(R+1\right)}(t_m(\Ds))$ 
is $(\tau,\beta,\alpha)$-robust.
That is, with probability $1-\beta$, for any $\tau$-corrupted $\Dsc$, such that $\dham(\Ds,\Dsc)\leq n\tau$, it holds that, 
$
\ltwo{\A_\mathrm{rob}(\Ds') - \mu} \le \alpha.
$
\end{proposition}
Using the above proposition, the proof of~\Cref{cor:sph-gauss-mean} is a straightforward application of~\Cref{thm:high-dim-rob-to-priv}.
\begin{proof}[Proof of~\Cref{cor:sph-gauss-mean}]
Consider the $\rho$-smooth-inverse-sensitivity mechanism $\smmech(\cdot;\A_\mathrm{rob})$ with norm $\norm{\cdot}=\ltwo{\cdot}$, $\A_{\mathrm{rob}}(\Ds)=\Pi_{\ball\left(R+1\right)}(t_m(\Ds))$ and 
$\rho=\alpha_0=C_0\sqrt{(d+\log(1/\beta))/n}$, as in~\Cref{prop:sph-accr-tukey} above. We apply~\Cref{thm:high-dim-rob-to-priv} to obtain a bound on the error of $\smmech(\Ds;\A_{\mathrm{rob}})$. Let $$\tau=\frac{2d\log\left(\frac{R+1}{\alpha_0}+1\right)+2\log(\frac{1}{\beta})}{n\diffp}.$$ 
Assume $\tau\le 0.05$ and $\alpha_0\le 0.05$, which we will confirm later. By~\Cref{prop:sph-accr-tukey}, $\A_\mathrm{rob}$ is $(\tau,\beta,\alpha)$-robust for 
\begin{align*}
\alpha 
& = 7C_0\cdot\sqrt{\frac{d+\log(\frac{1}{\beta})}{n}} + 7\left(\frac{2d\log\left(\frac{R+1}{\alpha_0}+1\right)+2\log(\frac{1}{\beta})}{n\diffp}\right)\\
&\le C'\cdot\left(\sqrt{\frac{d+\log(\frac{1}{\beta})}{n}} + \frac{d\log\left(\frac{Rn}{d}\right)+\log(\frac{1}{\beta})}{n\diffp}\right),
\end{align*}
for constant $C'=42C_0$. Notice that $\alpha_0\le \alpha$.
Therefore, by~\Cref{thm:high-dim-rob-to-priv}, it holds that, with probability at least $1-2\beta$,
$
\ltwo{\smmech(\Ds;\A_\mathrm{rob}) - \mu} \le 4\alpha\le C\cdot\left(\sqrt{\frac{d+\log(\frac{1}{\beta})}{n}} + \frac{d\log\left(\frac{Rn}{d}\right)+\log(\frac{1}{\beta})}{n\diffp}\right),
$
for $C=168C_0$. By assumption, $n$ is sufficiently large so that the latter is less than $1$ and as such, it also ensures that $\alpha_0\le 0.05$ and $\tau\le 0.05$.
The proof is complete by rescaling $\beta\gets \beta/2$ and adjusting the constants.
\end{proof}

\subsubsection{Unknown Covariance}
We now move to the more general task of Gaussian mean estimation with unknown mean $\mu$ and covariance $\Sigma$, but with known \emph{a priori} bounds $R,\kappa$ such that $\mu\in \ball^d(R)$ and $\id\preceq \Sigma \preceq \kappa\id$. The error metric is the affine-invariant \emph{Mahalanobis distance} with respect to $\Sigma$, defined by $\norm{\hat{\mu}-\mu}_{\Sigma}\defeq \sqrt{(\hat{\mu}-\mu)^\top \Sigma^{-1}(\hat{\mu}-\mu)}$. In~\Cref{cor:gauss-mean}, we show that via our transformation, we retrieve known error bounds for Gaussian mean estimation with known parameters $R,\kappa$ under pure DP~\citep{BunKSW19, LiuKK021}.\footnote{These results are stated for $\Sigma=\id$, but can be extended to the case of unknown $\Sigma$ such that $\id\preceq\Sigma\preceq \kappa\id$, achieving the same error as in~\Cref{cor:gauss-mean} up to logarithmic factors.}
\begin{corollary}[Gaussian mean]\label{cor:gauss-mean}
Let $\Ds = (S_1,\dots,S_n)$ where $S_i \simiid \normal(\mu,\Sigma)$ such that $\mu \in \ball(R)$ and $\id\preceq\Sigma\preceq \kappa\id$. Let $\diffp, \beta\in (0,1)$ and $C\ge1$ a known constant. Suppose $n$ is such that $\alpha\leq 1$ in~\Cref{eq:gauss-mean}.
There exists an $\diffp$-DP algorithm $\mc{M}$ such that, with probability at least $1-\beta$, has error $\norm{\mc{M}(\Ds) - \mu}_{\Sigma} \le \alpha$ for
\begin{equation}\label{eq:gauss-mean}
    \alpha=C\cdot\left(\sqrt{\frac{d+\log(\frac{1}{\beta})}{n}} + \frac{d\log\left(\frac{(R+\sqrt{\kappa})n}{d}\right)+\log(\frac{1}{\beta})}{n\diffp}\right).
\end{equation}
\end{corollary}
Again, we choose the projected Tukey median as our robust mechanism for this task. We state its guarantees for the Mahalanobis loss (proven in \Cref{app:facts}).
\begin{proposition}\label{prop:accr-tukey}
Let $\Ds = (S_1,\dots,S_n)$ where $S_i \simiid \normal(\mu,\Sigma)$ such that $\mu \in \ball(R)$ and $\id\preceq \Sigma \preceq \kappa\id$. Let $\beta\in(0,1)$, $\tau\leq 0.05$, and $\alpha_0=C_0\cdot\sqrt{(d+\log(1/\beta))/n}$ for known constant $C_0$. Suppose $n$ is such that $\alpha_0\leq 0.05$. Let $\alpha=7(\alpha_0+\tau)\leq 1$.
The projected Tukey median algorithm 
$\A_{\mathrm{rob}}(\Ds)=\Pi_{\ball\left(R+\sqrt{\kappa}\right)}(t_m(\Ds))$ 
is $(\tau,\beta,\alpha)$-robust with respect to the Mahalanobis loss. 
That is, with probability $1-\beta$, for any $\tau$-corrupted $\Dsc$, such that $\dham(\Ds,\Dsc)\leq n\tau$, it holds that
$
\norm{\A_\mathrm{rob}(\Ds') - \mu}_{\Sigma} \le \alpha.
$
\end{proposition}
Using the above proposition, the proof of~\Cref{cor:gauss-mean} is a straightforward application of~\iflong\Cref{thm:high-dim-rob-to-priv-loss}\else\Cref{thm:high-dim-rob-to-priv-loss-app}\fi.
\begin{proof}[Proof of~\Cref{cor:gauss-mean}]
We let $L(u,v)=\norm{u-v}_{\Sigma}$ be the loss function. As a norm, $L$ satisfies the triangle inequality. Moreover, $\forall s,t\in\R^d ~ L(s,t)\leq c_L\ltwo{s-t}$ for $c_L=1$ since $\id\preceq\Sigma$ (by~\Cref{prop:distance-wrt-matrices} in~\Cref{app:facts}). 
Consider the $\rho$-smooth-inverse-sensitivity mechanism $\smmech(\cdot;\A_\mathrm{rob})$ with norm $\norm{\cdot}=\ltwo{\cdot}$, $\A_{\mathrm{rob}}(\Ds)=\Pi_{\ball\left(R+\sqrt{\kappa}\right)}(t_m(\Ds))$ and 
$\rho=\alpha_0=C_0\sqrt{(d+\log(1/\beta))/n}$. We apply\iflong~\Cref{thm:high-dim-rob-to-priv-loss}\else~\Cref{thm:high-dim-rob-to-priv-loss-app}\fi~to obtain a bound on the mechanism's error with respect to $L$. Let $$\tau=\frac{{2d\log\left(\frac{R+\sqrt{\kappa}}{\alpha_0}+1\right)+2\log(\frac{1}{\beta})}}{n\diffp}.$$ 
Assume $\tau\le 0.05$ and $\alpha_0\le 0.05$, which we will confirm later. By~\Cref{prop:sph-accr-tukey}, $\A_\mathrm{rob}$ is $(\tau,\beta,\alpha)$-robust for 
\begin{align*}
\alpha 
& = 7C_0\sqrt{\frac{d+\log(\frac{1}{\beta})}{n}} + 7\frac{2d\log\left(\frac{R+\sqrt{\kappa}}{\alpha_0}+1\right)+2\log(\frac{1}{\beta})}{n\diffp}\\
&\le C'\cdot\left(\sqrt{\frac{d+\log(\frac{1}{\beta})}{n}} + \frac{d\log\left(\frac{(R+\sqrt{\kappa})n}{d}\right)+\log(\frac{1}{\beta})}{n\diffp}\right),
\end{align*}
for constant $C'=28C_0$. Notice that $\alpha_0\leq \alpha$.
Therefore, by~\iflong\Cref{thm:high-dim-rob-to-priv-loss}\else\Cref{thm:high-dim-rob-to-priv-loss-app}\fi, it holds that, with probability at least $1-2\beta$,
\begin{equation*}
L(\smmech(\Ds;\A_\mathrm{rob}), \mu)=\norm{\smmech(\Ds;\A_\mathrm{rob})-\mu}_{\Sigma} \le 4 \alpha\le C\cdot\left(\sqrt{\frac{d+\log(\frac{1}{\beta})}{n}} + \frac{d\log\left(\frac{(R+\sqrt{\kappa})n}{d}\right)+\log(\frac{1}{\beta})}{n\diffp}\right),
\end{equation*}
for $C=112C_0$. By assumption, $n$ is sufficiently large so that the latter is less than $1$, and as such, it also ensures that $\alpha_0\le 0.05$ and $\tau\le 0.05$. 
The proof is complete by rescaling $\beta\gets \beta/2$ and adjusting the constants.
\end{proof}

\subsection{Covariance Estimation}\label{sec:gauss-covariance}
Given dataset $\Ds\simiid \normal(0,\Sigma)^n$, where $\id\preceq \Sigma\preceq \kappa\id$, our goal is to return an estimate $\hat{\Sigma}\in \R^{d\times d}$, with small error, measured by the \emph{relative Frobenius norm}: $\lfro{\Sigma^{-1/2}\hat{\Sigma}\Sigma^{-1/2}-\id}$.
The task of covariance estimation for Gaussian distributions has been extensively studied both under robustness and differential privacy, and is particularly useful as a first step for learning a Gaussian distribution in total variation distance 
(see e.g. Corollary 2.14 in~\cite{DiakonikolasKKLMS16}). 
Note that the fact that the distribution is assumed to be zero-mean is w.l.o.g., as the general case can be reduced to the zero-mean case up to constant factors in the error, by letting the difference between a pair of nonzero-mean samples be a single zero-mean sample. 

In~\Cref{cor:gauss-covariance}, we show that via our transformation, we retrieve the optimal known error bounds for Gaussian covariance estimation with known parameter $\kappa$ under pure DP~\citep{BunKSW19, AdenAliAK21}.\footnote{Knowledge of parameter $\kappa$ is necessary for this task under pure DP~\citep{BunKSW19, AlabiKTVZ22}.}
\begin{corollary}[Gaussian covariance]\label{cor:gauss-covariance}
Let $\Ds = (S_1,\dots,S_n)$ where $S_i \simiid \normal(0,\Sigma)$ such that $\id\preceq\Sigma\preceq \kappa\id$. Let $\diffp, \beta\in (0,1)$. Suppose $n$ is such that $\alpha\leq 1$ in~\Cref{eq:gauss-covariance}.
There exists an $\diffp$-DP algorithm $\mc{M}$ such that, with probability at least $1-\beta$, has error $\lfro{\Sigma^{-1/2}\mc{M}(\Ds)\Sigma^{-1/2} - \id} \le \alpha$ for
\begin{equation}\label{eq:gauss-covariance}
    \alpha=O\left(\left(\sqrt{\frac{d^2}{n}} +\frac{d^2}{n\diffp}\right)\cdot \mathrm{polylog}(n\kappa/\beta)\right).
\end{equation}
\end{corollary}

There are several algorithms in the robust statistics literature that achieve near-optimal bounds for robust covariance estimation of Gaussian distributions, which can serve as a good instantiation of our transformation. The next theorem states the robust accuracy guarantees of the algorithm proposed in~\citep{DiakonikolasKKLMS17}.\footnote{This algorithm as well as other alternatives~\citep{ChengDGW19, LiY20} are computationally efficient. It is possible that by using a computationally inefficient algorithm we would achieve smaller error up to logarithmic factors, but since we are not aiming to optimize for those factors, we chose the clearer statement from~\citep{DiakonikolasKKLMS17}.} \lz{this uses a suboptimal robust algorithm but I can't find a better statement}

\begin{theorem}[\citep{DiakonikolasKKLMS17}]\label{thm:robust-covariance}
Let $\Ds = (S_1,\dots,S_n)$ where $S_i \simiid \normal(0,\Sigma)$. Let $\beta\in(0,1), \tau\in(0,1)$. Suppose $n\geq \Omega\left(\frac{d^2\log^5(d/\tau\beta)}{\tau^2}\right)$. Let
$
\alpha'=O\left(\tau\log\left(1/\tau\right)\right).
$
There exists algorithm $\A_{\mathrm{rob}}$ which is $(\tau,\beta,\alpha')$-robust. That is, with probability $1-\beta$, for any $\tau$-corrupted $\Dsc$, such that $\dham(\Ds,\Dsc)\leq n\tau$, it returns matrix $\A_{\mathrm{rob}}(\Dsc)=\hat{\Sigma}$ such that
$
    \lfro{\Sigma^{-1/2}\hat{\Sigma}\Sigma^{-1/2}-\id}\leq \alpha'.
$
\end{theorem}

\begin{proof}[Proof of~\Cref{cor:gauss-covariance}]
We will run the $\rho$-smooth-inverse-sensitivity mechanism over vectors $\R^D$, $D=d^2$, with $\norm{\cdot}=\ltwo{\cdot}$. We denote by $\mathrm{vec}(V)\in\R^{d^2}$ the \emph{flattening} of a matrix $V\in\R^{d\times d}$, so that if $\mathrm{vec}(V)=v$, then $V_{i,j}=v_{d(i-1)+j}$. Then $\ltwo{\mathrm{vec}(U)-\mathrm{vec}(V)}=\lfro{U-V}$. Let $\A$ be the robust algorithm established in~\Cref{thm:robust-covariance}. We will instantiate our transformation with $\A_{\mathrm{rob}}(\Ds)=\Pi_{\ball^D(R')} (\mathrm{vec}(\A(\Ds)))$ for $R'=2\sqrt{d}\kappa$, that is, after flattening the output of $\A$, we take its euclidean projection on the $D=d^2$-dimensional ball of radius $R'$ in $\ell_2$ norm. Let $\hat{\Sigma}=\A(\Ds)$. We have that 
\begin{align*}
\ltwo{\mathrm{vec}(\hat{\Sigma})}
& = \lfro{\hat{\Sigma}} \\
& \le \lfro{\Sigma}+\lfro{\hat{\Sigma}-\Sigma} \tag{triangle inequality}\\
& = \lfro{\Sigma} + \lfro{\Sigma\Sigma^{-1}(\hat{\Sigma}-\Sigma)} \\
& \le \lfro{\Sigma} + \lfro{\Sigma}\cdot \lfro{\Sigma^{-1/2}(\hat{\Sigma}-\Sigma)\Sigma^{-1/2}} \tag{$\lfro{\cdot}$ sub-multiplicative, rotation-invariant} \\
& \le  \sqrt{d}\kappa\left(1+\lfro{\Sigma^{-1/2}(\hat{\Sigma}-\Sigma)\Sigma^{-1/2}}\right) \tag{$\lfro{\cdot}\leq \sqrt{d}\ltwo{\cdot}$}.
\end{align*}
By~\Cref{thm:robust-covariance}, the latter is at most $\sqrt{d}\kappa(1+\alpha')$ with probability $1-\beta$. Suppose $\alpha'\le 1$, which we will confirm last. Thus, with probability $1-\beta$, the projection on the euclidean ball with radius $R'=2\sqrt{d}\kappa$ will not affect the output of the algorithm and $\A_\mathrm{rob}$ will have the same accuracy guarantees as stated in~\Cref{thm:robust-covariance}. 

We will let the loss function $L:(\R^{d\times d})^2\to \R$ be $L(U,V)=\lfro{\Sigma^{-1/2}(U-V)\Sigma^{-1/2}}$ over matrices $U,V$. Our goal is then to return a matrix $U$ with small error $L(U,\Sigma)$.\footnote{We can straightforwardly convert any vector $v\in\R^{d^2}$ to a unique matrix $V\in\R^{d\times d}$ such that $v=\mathrm{vec}(V)$.} Note that $L$ satisfies the triangle inequality since the Frobenius norm does. 
For all $u,v\in \R^{d^2}$, let $V,U\in\R^{d\times d}$ denote their corresponding matrices. 
We have that
$L(U,V) = \lfro{\Sigma^{-1}(U-V)} \le \lfro{(U-V)} = \ltwo{u-v}$, since $\Sigma^{-1}\preceq \id$ and $\lfro{\cdot}$ is monotone. 
It follows that $L$ satisfies all the requirements of~\iflong\Cref{thm:high-dim-rob-to-priv-loss}\else\Cref{thm:high-dim-rob-to-priv-loss-app}\fi. 

Let $\alpha_0=O(\sqrt{(d^2+\log(1/\beta))/n})<1$, by assumption on $n$. We take $\tau$ which satisfies both $\tau=\Omega\left( \frac{2D\log(R'/\alpha_0+1)+2\log(1/\beta)}{n\diffp}\right)=\Omega\left(\frac{d^2\log(\kappa n/d)+\log(1/\beta)}{n\diffp}\right)$ (required by~\iflong\Cref{thm:high-dim-rob-to-priv-loss}\else\Cref{thm:high-dim-rob-to-priv-loss-app}\fi) and $\tau=\Omega\left(\sqrt{\frac{d^2\log^5(d/\tau\beta)}{n}}\right)$ (required by~\Cref{thm:robust-covariance}). Then $\A_{\mathrm{rob}}$ is $(\tau,\beta,\alpha')$-robust with 
$$\alpha'=O\left(\left(\sqrt{\frac{d^2}{n}} +\frac{d^2}{n\diffp}\right)\cdot \mathrm{polylog}(n\kappa/\beta)\right). 
$$
We then have that $\smmech(\cdot,\A_{\mathrm{rob}})$ with $\rho=\alpha_0$ is $\diffp$-DP 
and with probability at least $1-2\beta$, returns matrix $\hat{V}$, which has error
$$L(\hat{V},\Sigma)=\lfro{\Sigma^{-1/2}\hat{V}\Sigma^{1/2}-\id}\leq 4\alpha'=\alpha.$$
By assumption, $n$ is large enough so that $\alpha\le 1$ and as such $\alpha'<1$ as well. The statement follows by rescaling $\beta\gets\beta/2$ and adjusting the constants.
\end{proof}

\subsection{Linear Regression}\label{sec:gauss-linear}
In this section, we apply our transformation to obtain an algorithm for linear regression for Gaussian data under pure DP. To the best of our knowledge, ~\Cref{cor:gauss-linear} gives the first (computationally inefficient) algorithm for pure DP which achieves the optimal error rate up to logarithmic factors for Gaussian distributions. \citet{LiuKO22} gave the analogous result under approximate DP. 
\begin{corollary}[Gaussian Linear Regression]\label{cor:gauss-linear}
Let $\Ds = (S_1,\dots,S_n)$ where for all $i\in[n]$, $S_i =(X_i,y_i)\in\R^d\times \R$ is generated by a linear model $y_i=X_i^\top \theta + \eta_i$ for some unknown $\theta\in \ball^d(R)$, where $X_i\simiid \normal(0,\Sigma)$, $\id\preceq\Sigma\preceq \kappa\id$, and $\eta_i\simiid \normal(0,\sigma^2)$, independent from $X_i$.
Let $\diffp, \beta\in (0,1)$. Let 
\begin{equation}\label{eq:gauss-linear}
   \alpha=C\sigma\left(\sqrt{\frac{d+\log(\frac{1}{\beta})}{n}} +\frac{d\log\left(\frac{(R/\sigma+\kappa)n}{d}\right)+\log(\frac{1}{\beta})}{n\diffp}\right),
\end{equation} for a known constant $C>0$.
Suppose $n$ is such that $\alpha/\sigma\le c$ for a known constant $c\in(0,1)$. Then there exists an $\diffp$-DP algorithm $\mc{M}$ such that, with probability at least $1-\beta$, returns $\mc{M}(\Ds)=\hat{\theta}$ such that $\ltwo{\Sigma^{-1/2}(\hat{\theta}-\theta)}\leq \alpha.$
\end{corollary}

Since the running time of the robust algorithm is not the bottleneck for the computational complexity of our proposed approach, we will instantiate our transformation with the (computationally inefficient) robust linear regression algorithm from~\citep{Gao20}. This algorithm achieves the information-theoretic optimal error for Gaussian distributions and is based on the notion of multivariate regression depth, similar to the Tukey depth we used for Gaussian mean estimation in~\Cref{sec:gauss-mean}.\footnote{The result is stated for the weaker Huber's contamination model, but it holds under the strong contamination model as well.}
\begin{theorem}[Theorem 3.2,~\citep{Gao20}]\label{thm:robust-linear}
Consider the setting of~\Cref{cor:gauss-linear}. 
Let $\beta\in(0,1), \tau\in(0,1)$. Suppose $n$ and $\tau$ are such that $\tau+\sqrt{d/n}<c$ for a known constant $c\in(0,1)$. Then there exists constant $C'>0$ and algorithm $\A_{\mathrm{rob}}$ which is $(\tau,\beta,\alpha')$-robust, for
\begin{equation}\label{eq:robust-linear}
\alpha'=C'\sigma\left(\sqrt{\frac{d+\log(\frac{1}{\beta})}{n}}+\tau\right).
\end{equation}
That is, with probability $1-\beta$, for any $\tau$-corrupted $\Dsc$, such that $\dham(\Ds,\Dsc)\leq n\tau$, it returns $\A_{\mathrm{rob}}(\Dsc)=\hat{\theta}\in\R^d$ such that
$
\ltwo{\Sigma^{-1/2}(\hat{\theta}-\theta)}\leq \alpha'.
$
\end{theorem}

\begin{proof}[Proof of~\Cref{cor:gauss-linear}]
We will run the $\rho$-smooth-inverse-sensitivity mechanism in $\R^d$ with $\norm{\cdot}=\ltwo{\cdot}$. Let $\A$ be the robust algorithm established in~\Cref{thm:robust-linear}. We will instantiate our transformation with $\A_{\mathrm{rob}}(\Ds)=\Pi_{\ball(R')} (\A(\Ds))$ for $R'=R+\sigma\sqrt{\kappa}$, that is, we take the euclidean projection of $\A(\Ds)$ on the ball of radius $R'$ in $\ell_2$ norm. Let $\hat{\theta}=\A(\Ds)$. We have that 
\begin{align*}
\ltwo{\hat{\theta}}
& \le \ltwo{\theta}+\ltwo{\hat{\theta}-\theta} \tag{triangle inequality}\\
& \leq R + \ltwo{\Sigma^{1/2}}\ltwo{\Sigma^{-1/2}(\hat{\theta}-\theta)} \\
& \le R+\sqrt{\kappa}\ltwo{\Sigma^{-1/2}(\hat{\theta}-\theta)}.
\end{align*}
Let $\alpha_0=C'\sigma\sqrt{(d+\log(1/\beta))/n}$ for $C'>0$ as in~\Cref{thm:robust-linear} and   $$\tau=\frac{2d\log(R'/\alpha_0+1)+2\log(1/\beta)}{n\diffp}.$$ Assume that $n$ is such that $C'\left(\sqrt{\frac{d+\log(\frac{1}{\beta})}{n}}+\tau\right)<c$ for $c\in(0,1)$ and for $C'>0$ as in~\Cref{thm:robust-linear}, which we will confirm last. Then, the conditions of~\Cref{thm:robust-linear} are satisfied, and with probability $1-\beta$, $R+\sqrt{\kappa}\ltwo{\Sigma^{-1/2}(\hat{\theta}-\theta)}\le R+\sqrt{\kappa}\alpha'\le R+\sigma\sqrt{\kappa}= R'$ and the projection will not affect the output of the algorithm $\A_\mathrm{rob}$. 

We will let the loss function $L:(\R^d)^2\to \R$ be $L(u,v)=\ltwo{\Sigma^{-1/2}(u-v)}$. Our goal is then to return a vector $u$ with small error $L(u,\theta)$. Note that $L$ satisfies the triangle inequality. 
For all $u,v\in \R^{d}$, we have that
$L(u,v) = \ltwo{\Sigma^{-1/2}(u-v)} \le \ltwo{u-v}$, since $\Sigma^{-1/2}\preceq \id$. 
It follows that $L$ satisfies all the requirements of~\iflong\Cref{thm:high-dim-rob-to-priv-loss}\else\Cref{thm:high-dim-rob-to-priv-loss-app}\fi. 
Thus, $\smmech(\cdot,\A_{\mathrm{rob}})$ with $\rho=\alpha_0<\alpha'$ is $\diffp$-DP 
and with probability at least $1-2\beta$, returns $\hat{u}$, which has error
$$L(\hat{u},\theta)=\ltwo{\Sigma^{-1/2}(\hat{u}-\theta)}\leq 4\alpha'.$$
That is, there exists $C>C'$, such that $\ltwo{\Sigma^{-1/2}(\hat{u}-\theta)}\le \alpha$, for 
$$\alpha= C\sigma\left(\sqrt{\frac{d+\log(\frac{1}{\beta})}{n}}+\frac{d\log((R+\sigma\sqrt{\kappa})n/(\sigma d))+\log(1/\beta)}{n\diffp}\right).$$ 
By assumption $n$ is sufficiently large so that the latter is smaller than $\sigma c$, and as such, it ensures that $C'\left(\sqrt{\frac{d+\log(\frac{1}{\beta})}{n}}+\tau\right)<c$ as well. The proof is complete by rescaling $\beta\gets\beta/2$ and adjusting the constants.
\end{proof}

    \subsubsection{Sparse Linear Regression}\label{sec:sparse-linear}
We now apply our transformation to obtain an algorithm for sparse linear regression for Gaussian data under pure DP, which, to the best of our knowledge, is the first (computationally inefficient) algorithm this case that achieves near-optimal error rate. When the solution is known to be $k$-sparse, our transformation allows us to improve the dependence on dimension from $d/n\diffp$ to $k \log d/(n\diffp)$ as we show in the next corollary.
\begin{corollary}[Sparse Linear Regression]\label{cor:sparse-linear}
Let $\Ds = (S_1,\dots,S_n)$ where for all $i\in[n]$, $S_i =(X_i,y_i)\in\R^d\times \R$ is generated by a linear model $y_i=X_i^\top \theta + \eta_i$ for some unknown $\theta\in \ball^d(R)$, $\norm{\theta}_0\le k$, where $X_i\simiid \normal(0,\Sigma)$, $\id\preceq\Sigma\preceq \kappa\id$, and $\eta_i\simiid \normal(0,\sigma^2)$, independent from $X_i$.
Let $\diffp, \beta\in (0,1)$. Let 
\begin{equation}\label{eq:sparse-linear}
   \alpha=C\sigma\left(\sqrt{\frac{k\log(\frac{ed}{k})+\log(\frac{1}{\beta})}{n}} +\frac{k\log(\frac{ed}{k})+k\log\left(\frac{(R/\sigma+\kappa)n}{d}\right)+\log(\frac{1}{\beta})}{n\diffp}\right),
\end{equation} for a known constant $C>0$.
Suppose $n$ is such that $\alpha/\sigma\le c$ for a known constant $c\in(0,1)$. Then there exists an $\diffp$-DP algorithm $\mc{M}$ such that, with probability at least $1-\beta$, returns $\mc{M}(\Ds)=\hat{\theta}$ such that $\ltwo{\Sigma^{-1/2}(\hat{\theta}-\theta)}\leq \alpha.$
\end{corollary}

We use the robust algorithm for sparse linear regression by~\cite{Gao20}.
\begin{theorem}[Theorem 3.2,~\citep{Gao20}]\label{thm:robust-linear-sparse}
Consider the setting of~\Cref{cor:sparse-linear}. 
Let $\beta\in(0,1), \tau\in(0,1)$. Suppose $n$ and $\tau$ are such that $\tau+\sqrt{k\log(ed/k)/n}<c$ for a known constant $c\in(0,1)$. Then there exists constant $C'>0$ and algorithm $\A_{\mathrm{rob}}$ which is $(\tau,\beta,\alpha')$-robust, for
\begin{equation}\label{eq:robust-linear-sparse}
\alpha'=C'\sigma\left(\sqrt{\frac{k\log(\frac{ed}{k})+\log(\frac{1}{\beta})}{n}}+\tau\right).
\end{equation}
That is, with probability $1-\beta$, for any $\tau$-corrupted $\Dsc$, such that $\dham(\Ds,\Dsc)\leq n\tau$, it returns $\A_{\mathrm{rob}}(\Dsc)=\hat{\theta}\in\R^d$ such that
$
\ltwo{\Sigma^{-1/2}(\hat{\theta}-\theta)}\leq \alpha'.
$
\end{theorem}
The proof of~\Cref{cor:sparse-linear} follows exactly the same steps as~\Cref{cor:gauss-linear}, but uses the slightly modified inverse-sensitivity mechanism for sparse estimation and its guarantees in~\Cref{thm:high-dim-rob-to-priv-sparse} instead of~\Cref{thm:high-dim-rob-to-priv-loss}.

\fi
\subsection{Principal Component Analysis}\label{sec:gauss-pca}
In this section, we apply our transformation to obtain a pure DP algorithm for PCA under Gaussian data. We note that the result holds as-is for subgaussian distributions more generally, because~\Cref{thm:robust-pca}~\citep{JambulapatiLT20} does.
We assume w.l.o.g.\ that the distribution is zero-mean\iflong, as in~\Cref{sec:gauss-covariance}\fi. 

To the best of our knowledge,~\Cref{cor:gauss-pca} gives the first (computationally inefficient) algorithm for pure DP with error $\tilde{O}(\frac{d}{n\diffp})$. 
PCA with a spectral gap has been studied under pure DP in~\citep{ChaudhuriSS13}, where the result can be translated to yield a suboptimal error of $\tilde{O}(\frac{d^2}{n\diffp})$. \lz{I am not 100\% sure this is true for their paper.} 
There is a long line of work that studies PCA under approximate DP~\citep{BlumDMN05, HardtR12, HardtR13, ChaudhuriSS13, KapralovT13, DworkTTZ14} with the recent result by~\citet{LiuKO22} achieving the optimal error of $\tilde{O}(\frac{d}{n\diffp})$ for subgaussian distributions.

\begin{corollary}[Gaussian PCA]\label{cor:gauss-pca}
Let $\Ds = (S_1,\dots,S_n)$ where $S_i \simiid \normal(0,\Sigma)$. 
Let $\diffp, \beta\in (0,1)$. Suppose $n$ is such that $\alpha\leq 1$ in~\Cref{eq:gauss-pca}.
There exists a constant $C>0$ and an $\diffp$-DP algorithm $\mc{M}$ such that, with probability at least $1-\beta$, returns unit vector $\mc{M}(\Ds)=v$ such that $1-\frac{v^\top\Sigma v}{\ltwo{\Sigma}}\leq \alpha$ for
\iflong
\begin{equation}\label{eq:gauss-pca}
    \alpha=C\left(\sqrt{\frac{d+\log(\frac{1}{\beta})}{n}} +\frac{d\log^2(\frac{n}{d})+\log(\frac{1}{\beta})\log(\frac{n}{d})}{n\diffp}\right).
\end{equation}
\else
\begin{equation}\label{eq:gauss-pca}
    \alpha=C\left(\sqrt{\frac{d+\log(\frac{1}{\beta})}{n}} +\frac{d\log^2(\frac{n}{d})+\log(\frac{1}{\beta})\log(\frac{n}{d})}{n\diffp}\right).
\end{equation}
\fi
\end{corollary}

We will instantiate our transformation with the robust PCA algorithm from~\citep{JambulapatiLT20}. An alternative with the same guarantees is returning the unit vector that minimizes the surrogate cost function proposed by~\citep{LiuKO22}.
\begin{theorem}[Theorem 1,~\citep{JambulapatiLT20}]\label{thm:robust-pca}
Let $\Ds = (S_1,\dots,S_n)$ where $S_i \simiid \normal(0,\Sigma)$. Let $\beta\in(0,1), \tau\in(0,1/2)$. 
Let
$$
\alpha'=C'\left(\sqrt{\frac{d+\log(\frac{1}{\beta})}{n}}+\tau\log\left(\frac{1}{\tau}\right)\right),
$$
for a known constant $C'>0$.
There exists algorithm $\A_{\mathrm{rob}}$ which is $(\tau,\beta,\alpha')$-robust. 
\iflong That is, with probability $1-\beta$, for any $\tau$-corrupted $\Dsc$, such that $\dham(\Ds,\Dsc)\leq n\tau$, it returns unit vector $\A_{\mathrm{rob}}(\Dsc)=u$ such that
$
    1-\frac{u^\top\Sigma u}{\ltwo{\Sigma}}\leq \alpha'.
$\fi
\end{theorem}

\begin{proof}[Proof of~\Cref{cor:gauss-pca}]
We define the following loss function $L(u,v)=\frac{u^\top\Sigma u}{\ltwo{\Sigma}}-\frac{v^\top\Sigma v}{\ltwo{\Sigma}}$. If $v_1$ is the top eigenvector of $\Sigma$, then \iflong $v_1^\top\Sigma v_1=\ltwo{\Sigma^{1/2}v_1}^2=\ltwo{\Sigma}$ and \fi our goal is to return vector $v$ with small error $L(v_1,v)$. 
Let $\alpha_0=C'\sqrt{\frac{d+\log(1/\beta)}{n}}$ and assume it is less than $1$, to be confirmed later. Then $L$ satisfies the triangle inequality and for all $u,v\in \ball(1+\alpha_0)\subset \ball(2)$,
\iflong
    \begin{align*}
    L(u,v)
    & = \frac{u^\top\Sigma u}{\ltwo{\Sigma}}-\frac{v^\top\Sigma v}{\ltwo{\Sigma}} \\
    & = \frac{\ltwo{\Sigma^{1/2}u}^2}{\ltwo{\Sigma}}-\frac{\ltwo{\Sigma^{1/2}v}^2}{\ltwo{\Sigma}} \\
    & = \frac{1}{\ltwo{\Sigma}}\left(\ltwo{\Sigma^{1/2}u}-\ltwo{\Sigma^{1/2}v}\right)\cdot \left(\ltwo{\Sigma^{1/2}u}+\ltwo{\Sigma^{1/2}v}\right) \\
    & \leq \ltwo{u-v}\cdot (\ltwo{u}+\ltwo{v}) \\
    & \leq 4\ltwo{u-v}.
    \end{align*}
\else
    \begin{align*}
    &L(u,v)= \frac{\ltwo{\Sigma^{1/2}u}^2}{\ltwo{\Sigma}}-\frac{\ltwo{\Sigma^{1/2}v}^2}{\ltwo{\Sigma}} \\
    & = \frac{\left(\ltwo{\Sigma^{1/2}u}-\ltwo{\Sigma^{1/2}v}\right)\cdot \left(\ltwo{\Sigma^{1/2}u}+\ltwo{\Sigma^{1/2}v}\right)}{\ltwo{\Sigma}} \\
    & \leq \ltwo{u-v}\cdot (\ltwo{u}+\ltwo{v}) \leq 4\ltwo{u-v}.
    \end{align*}
\fi
Let $\A_{\mathrm{rob}}:(\R^d)^n \to \sphere^{d-1}$ be the algorithm established by~\Cref{thm:robust-pca}, where $\sphere^{d-1}$ denotes the unit sphere. Then $L$ satisfies the requirements of~\Cref{thm:high-dim-rob-to-priv-loss}.
For $\tau= \frac{2d\log(n/d)+2\log(1/\beta)}{n\diffp}$, $\A_{\mathrm{rob}}$ is $(\tau,\beta,\alpha')$-robust with 
$$\alpha'=C'\left(\sqrt{\frac{d+\log(1/\beta)}{n}} +\frac{d\log^2(\frac{n}{d})+\log(\frac{1}{\beta})\log(\frac{n}{d})}{n\diffp}\right). 
$$
We then have that $\smmech(\cdot,\A_{\mathrm{rob}})$ with $\rho=\alpha_0$ is $\diffp$-DP 
and with probability $1-2\beta$, returns $v\in \ball(1+\alpha_0)$, such that
$$L(v_1,v)=1-\frac{v^\top\Sigma v}{\ltwo{\Sigma}}\leq 7\alpha'.$$

Let $\hat{v}=\frac{v}{\ltwo{v}}$ be the unit vector in the direction of $v$. We have that $L(v_1,\hat{v})=1-\frac{v^\top\Sigma v}{\ltwo{\Sigma}\ltwo{v}^2}$. If $\ltwo{v}\leq 1$ then $L(v_1,\hat{v})\leq L(v_1,v)$. Suppose $\ltwo{v}>1$.
\begin{align*}
L(v_1,\hat{v})
& = \frac{1}{\ltwo{v}^2}\left(\ltwo{v}^2-\frac{v^\top\Sigma v}{\ltwo{\Sigma}}\right)\\
& = \frac{1}{\ltwo{v}^2}\left((\ltwo{v}^2-1) +L(v_1,v)\right)\\
& \le \frac{\ltwo{v}^2-1}{\ltwo{v}^2}+L(v_1,v) \tag{since $\ltwo{v}>1$}\\
& \le \frac{\alpha_0(\alpha_0+2)}{(\alpha_0+1)^2}+L(v_1,v) \tag{since $(x-1)/x \nearrow$}\\
& \le 2\alpha_0+L(v_1,v) \leq 9\alpha',
\end{align*}
since $\alpha_0\leq \alpha'$. 
Therefore, we return a unit vector $\hat{v}$ with $1-\frac{\hat{v}^\top\Sigma\hat{v}}{\ltwo{\Sigma}}\leq \alpha$ for $\alpha=9\alpha'$. By assumption $n$ is sufficiently large so that $\alpha\le 1$, and as such $\alpha'<1$. The proof is complete by rescaling $\beta\gets \beta/2$ and adjusting the constants.
\end{proof}

\iflong
\section{Conclusions and Future Work}\label{sec:future}
We gave the first black-box transformation that converts an arbitrary robust algorithm into a differentially private one with similar accuracy. We proved that this transformation gives an optimal strategy for designing a differentially private algorithm for low-dimensional tasks, and that the minimax errors for robustness and privacy are equivalent for these tasks.

We also showed that this transformation often gives near-optimal error rates for several canonical high-dimensional tasks under (sub)Gaussian distributions (including under sparsity assumptions) and expect that it achieves similar results under other families of distributions, such as heavy-tailed. A natural question is to explore the conditions under which this or another black-box transformation yields private algorithms with optimal error in high dimensions. 

A drawback of our transformation is that it produces a computationally inefficient algorithm, even if the robust algorithm we instantiate it with is computationally efficient.\citet{HopkinsKMN22} use the sum-of-squares paradigm, to make this transformation computationally efficient specifically for the task of Gaussian estimation in TV distance, an approach that has been recently successful when applied to several problems~\citep{HopkinsKM22, AshtianiL22, KothariMV22, AlabiKTVZ22}. Another interesting direction is to explore under which conditions there exists a transformation that is simultaneously computationally efficient and gives differentially private algorithms with optimal error for all tasks.

\subsection*{Acknowledgements}

We thank Chao Gao for helpful discussions and clarifications regarding results in the robustness literature.

\fi
\ifshort
\else
    \printbibliography
\fi

\newpage
\appendix
\onecolumn
\section{Additional Proofs for Transformations}\label{app:transfproofs}
\subsection{Randomized to Deterministic Robust Algorithm}\label{app:rand-to-det}
Here, we present a transformation from a randomized algorithm $\A$ to a \emph{deterministic} robust algorithm whose error and failure probability are larger by a factor of 2.
Intuitively, \Cref{alg:rand-to-det} finds a small ball where the randomized algorithm has the largest density, and returns its center. This transformation is computationally inefficient because it requires running the randomized algorithm with all possible choices of random coins. 

\begin{algorithm}
	\caption{Randomized-to-Deterministic Robust}
	\label{alg:rand-to-det}
	\begin{algorithmic}[1]
		\REQUIRE $\Ds = (S_1,\dots,S_n)$, Algorithm $\A$, accuracy $\alpha$.
	\STATE Let $P_\Ds$ denote the probability distribution of $\A(\Ds)$ over the randomness of the algorithm
        \STATE Find the center $v\opt$ of a ball $\ball(v) = \{ u\in\R^d: \norm{u-v} \le \alpha\}$ of radius $\alpha$ that maximizes $P_\Ds(\ball(v))$
        \STATE Return $v\opt$
	\end{algorithmic}
\end{algorithm}

We have the following guarantees for the transformation in~\Cref{alg:rand-to-det}.
\begin{theorem}[Randomized-to-deterministic-robust]
\label{thm:rand-to-det}
  Let $\Ds = (S_1,\dots,S_n)$ where $S_i \simiid P$. Let $\tau, \beta\in(0,1)$.
  Let $\A$ be a $(\tau,\beta,\alpha)$-robust algorithm for estimating the statistic $\mu$. Then~\Cref{alg:rand-to-det} is $\left(\tau,2\beta,2\alpha\right)$-robust.
\end{theorem}
\begin{proof}
Let $\Ds = (S_1,\dots,S_n)$ where $S_i \simiid P$ and let $\Dsc$ be a $\tau$-corrupted version of $\Ds$, that is, $\dham(\Ds,\Dsc) \le n\tau$. We show that running~\Cref{alg:rand-to-det} over $\Dsc$ returns an accurate estimate with high probability over $\Ds$.
Let $\ball(\mu) = \{ u\in\R^d: \norm{u - \mu} \le \alpha\}$ be the ball of radius $\alpha$ around $\mu$, and let $\ball(v\opt)$ be the ball that maximizes $P_\Dsc(\ball(v))$. 
Let $E = \{\Ds: \forall \Dsc~\dham(\Dsc,\Ds)\leq n\tau, P_\Dsc(\ball(\mu)^c)\le 1/2 \}$ denote the set of good input datasets $\Ds$ such that for all $\tau$-corrupted $\Dsc$, the robust algorithm $\A$ returns bad answers with probability less than $1/2$. We show that if $\Ds \in E$ then~\Cref{alg:rand-to-det} returns an accurate answer for any $\tau$-corrupted $\Dsc$. 
Indeed, if $\Ds\in E$, then for any $\tau$-corrupted $\Dsc$, $P_\Dsc(\ball(\mu))>1/2$. 
Moreover, the definition of $\ball(v\opt)$  implies 
\begin{align*}
P_\Dsc(\ball(v\opt))\ge P_\Dsc(\ball(\mu)) > 1/2.
\end{align*}
As a result, we have that $\ball(v\opt) \cap \ball(\mu) \neq \emptyset$. Let $u \in \ball(v\opt) \cap \ball(\mu)$. We have that
\begin{align*}
\norm{v\opt - \mu} \le \norm{v\opt - u} + \norm{u - \mu} \le 2\alpha.
\end{align*}
Thus, if $\Ds \in E$ then~\Cref{alg:rand-to-det} is $2\alpha$ accurate. It remains to show that $\Pr[\Ds \notin E] \le 2 \beta$. This follows from the fact that $\A$ has failure probability $\beta$:
\begin{align*}
\beta 
& \ge \Pr_{\Ds,\A}[\exists \Dsc: \dham(\Dsc,\Ds) \text{ and } \A(\Dsc) \in \ball(\mu)^c] \\ 
& \ge \Pr_{\Ds}[\Ds \notin E]\cdot \Pr_{\Ds,\A}\left[\exists \Dsc: \dham(\Dsc,\Ds)\le n\tau \text{ and } \A(\Dsc) \in \ball(\mu)^c \mid \Ds\notin E\right]\\
& = \Pr_{\Ds}[\Ds \notin E]\cdot \E_{\Ds,\A}\left[\max_{\Dsc: \dham(\Dsc,\Ds)\le n\tau} \mathbbm{1}\{\A(\Dsc)\in \ball(\mu)^c\} \mid \Ds\notin E\right]\\
& \ge \Pr_{\Ds}[\Ds \notin E]\cdot \E_{\Ds}\left[\max_{\Dsc: \dham(\Dsc,\Ds)\le n\tau} \E_{\A}[\mathbbm{1}\{\A(\Dsc)\in \ball(\mu)^c\}] \mid \Ds\notin E\right] \tag{by Jensen's inequality}\\
& = \Pr_{\Ds}[\Ds \notin E]\cdot \E_{\Ds}\left[\max_{\Dsc: \dham(\Dsc,\Ds)\le n\tau} P_\Dsc(\ball(\mu)^c) \mid \Ds\notin E\right]\\
& > \Pr_{\Ds}[\Ds \notin E]\cdot \frac{1}{2}.
\end{align*}
The claim follows.
\end{proof}

\subsection{Robust-to-Private Transformation for General Loss}\label{app:gen-loss}
In the statement of~\Cref{thm:high-dim-rob-to-priv}, the error is measured in some norm $\norm{\cdot}$, the range of the robust algorithm $\A_\text{rob}$ is bounded in the same norm, and the smoothness $\rho$ allowed in the inverse sensitivity score function (\Cref{eq:def-sm-inv-sens}) is again bounded in the same norm. We can prove a more general theorem, where the last two norms are the same, but the error is instead measured with respect to a general loss function which satisfies the triangle inequality.
\begin{theorem}[Robust-to-private, general loss, restatement of~\Cref{thm:high-dim-rob-to-priv-loss}]
\label{thm:high-dim-rob-to-priv-loss-app}
  Let $\Ds = (S_1,\dots,S_n)$ where $S_i \simiid P$ such that $\mu(P) \in \R^d.$
  Let $\diffp, \beta \in (0,1)$. Let $L:(\R^d)^2\to \R$ be a loss function which satisfies the triangle inequality.
  Let $\A_\text{rob} : (\R^d)^n \to \{ t \in \R^d: \norm{t} \le R\}$ be a (deterministic) $(\tau,\beta,\alpha)$-robust algorithm with respect to $L$. Let $\alpha_0\leq \alpha$. Suppose $n$ is such that the smallest value $\tau$ satisfying~\Cref{eq:fraction-loss-app} is at most $1$. Suppose $\forall u,v\in\ball(R+\alpha_0) ~ L(u,v)\leq c_L\norm{u-v}$ for some constant $c_L$. 
  If 
  \begin{equation}\label{eq:fraction-loss-app}
  \tau \geq \frac{2\left(
    d\log\left(\frac{R}{\alpha_0}+1\right)+\log\frac{1}{\beta}\right)}{n\diffp},
    \end{equation}
    then Algorithm $\smmech(\Ds;\A_\text{rob})$ with $\rho = \alpha_0$ in norm $\norm{\cdot}$ is $\diffp$-DP and, with probability at least $1-2\beta$, has error
  \begin{equation*}
      L\left(\smmech(\Ds;\A_\text{rob}),\mu\right) \le (3+c_L)\alpha=O(\alpha).
  \end{equation*}
\end{theorem}

Before proving~\Cref{thm:high-dim-rob-to-priv-loss-app}, we need the equivalent of~\Cref{thm:ub-cont} for general loss functions. 
\begin{theorem}[Continuous functions, general loss]
\label{thm:ub-cont-loss}
  Let $\diffp, \beta \in (0,1)$, $\rho>0$. Let $f: \domain^n \to \range$ where $\range=\{ t \in \R^d: \norm{t} \le R\}$. Let $L:(\R^d)^2\to \R$ be a loss function which satisfies the triangle inequality and $\forall u,v\in\ball(R+\rho) ~ L(u,v)\leq c_L\norm{u-v}$ for some constant $c_L$. Suppose $n$ is larger than the smallest $K$ satisfying~\Cref{eq:corruptions-loss} below.
   Then for any $\Ds\in\domain^n$, with probability $1-\beta$, the $\rho$-smooth-inverse-sensitivity mechanism with norm $\norm{\cdot}$ has error
  \begin{equation*}
      L\left(\smmech(\Ds;f),f(\Ds)\right) \le \modcont_\func^L\left(\Ds;K\right)+c_L\rho,
  \end{equation*}
  where 
  \begin{equation}\label{eq:corruptions-loss}
  K\geq\frac{2d \log(R/\rho+1)+2\log(1/\beta)}{\diffp}
  \end{equation}
  and $\modcont_\func^L\left(\Ds;K\right)=\sup_{\Dsc:\dham(\Ds,\Dsc)\leq K} L\left(f(\Ds),f(\Dsc)\right)$ denotes the local modulus of continuity of $f$ at $\Ds$ with respect to loss function $L$.
\end{theorem}
\begin{proof}
    We define the good set of outputs $A = \{ t \in \R^d: \sminvmodcont(\Ds;t) \le K\}$, where $\sminvmodcont(\Ds;t)$ is defined with respect to norm $\norm{\cdot}$ and $K\in\naturals$. By the definition of $\rho$-smooth-inverse sensitivity, for any $t\in A$, there exists $s \in \range$ with $\invmodcont(\Ds;s) = K$ and $\norm{s-t} \le \rho$. We will show that $\Pr[\smmech(\Ds) \notin A] \le \beta$ for sufficiently large $K$. This implies the desired upper bound as we have that for $t \in A$
    \begin{align*}
   L(t,\func(\Ds))
         & \le L(t,s)+L(s,\func(\Ds)) \tag{by triangle inequality}\\
         & \le c_L\norm{t-s} + \modcont_\func^L(\Ds;K) \tag{since $s\in\ball(R)$, $\norm{t-s}\leq \rho$, so $s,t\in\ball(R+\rho)$}\\
         & \le c_L\rho + \modcont_\func^L(\Ds;K).
    \end{align*}
    Now we upper bound $\Pr[\smmech(\Ds;\func) \notin A]$.  First, note that $\sminvmodcont(\Ds;u) = 0$ for $u$ such that $\norm{u-\func(\Ds)} \le \rho$. This implies that for any $t$ such that $\sminvmodcont(\Ds;t) \ge K$, the density is upper bounded by
    \begin{equation*}
    \pi_{\Ds}(t) \le  \frac{e^{-K\diffp/2}}{\int_{u : \norm{u-\func(\Ds)} \le \rho} du}
    \end{equation*}
    Overall, this implies that 
    \begin{align*}
    \Pr[\smmech(\Ds;\func) \notin A] 
        & \le e^{-K\diffp/2} \frac{\int_{u: \norm{u} \le R + \rho}  du}{\int_{u : \norm{u-\func(\Ds)} \le \rho} du} \\
        & \le e^{-K\diffp/2} (R/\rho + 1)^d.
    \end{align*}
    Setting $K\geq \frac{2d \log(R/\rho+1)+2\log(1/\beta)}{\diffp}$, 
    we get that $\Pr(\smmech(\Ds;\func) \notin A)  \le \beta$.
\end{proof}

We are now ready to prove~\Cref{thm:high-dim-rob-to-priv-loss-app}.
\begin{proof}[Proof of~\Cref{thm:high-dim-rob-to-priv-loss-app}]
    First note that the claim about privacy is immediate from the guarantees of the smooth-inverse-sensitivity mechanism. Now we prove utility. Let $K = \frac{2\left(
    d\log\left(\frac{R}{\alpha_0}+1\right)+\log\frac{1}{\beta}\right)}{\diffp}$. 
    The error of $\smmech(\Ds,\A_\text{rob})$ is then bounded as follows:
    \begin{align*}
        & L\left(\smmech(\Ds;\A_\text{rob}), \mu\right) \\
        & \le L\left(\A_\text{rob}(\Ds),\mu\right)+ L\left(\smmech(\Ds;\A_\text{rob}), \A_\text{rob}(\Ds)\right) \tag{by triangle inequality}\\
        & \le L\left(\A_\text{rob}(\Ds),\mu\right)+\sup_{\Dsc: \dham(\Dsc,\Ds)\leq K}L\left(\A_{\text{rob}}(\Dsc),\A_{\text{rob}}(\Ds)\right) + c_L\alpha_0 \tag{w.p. $1-\beta$ by~\Cref{thm:ub-cont-loss}}\\
        & \le 2 L\left(\A_\text{rob}(\Ds),\mu\right) + \sup_{\Dsc: \dham(\Dsc,\Ds)\leq K}L\left(\mu, \A_\text{rob}(\Dsc)\right) +c_L\alpha_0 \tag{by triangle inequality}
    \end{align*}

Recall that, by assumption, $\A_\text{rob}$ is $(\tau,\beta,\alpha)$-robust for $\tau\geq K/n$, and $\alpha_0\leq \alpha$. Thus, with probability $1-\beta$, $L\left(\A_\text{rob}(\Dsc),\mu\right)\leq \alpha$ for any $\tau$-corrupted dataset $\Dsc$. By union bound, we have that with probability $1-2\beta$, $L\left(\smmech(\Ds;\A_\text{rob}), \mu\right)\leq 3\alpha+c_L\alpha_0\leq (3+c_L)\alpha=O(\alpha)$.
\end{proof}

\subsection{Robust-to-Private Transformation for Sparse Estimators}\label{sec:sparse}
In this section, we extend our transformation to work for $k$-sparse statistical estimation problems with improved dependence on the dimension.
To this end, we define a variant of the smooth inverse sensitivity which is non-zero only for $k$-sparse outputs, 
\newcommand{\lens}{\invmodcont^{\mathsf{sp}}}
\begin{align*}\label{eq:def-sparse-inv-sens}
    \lens(\Ds;t) =
    \begin{cases}
    \inf_{s \in \R^d: \norm{s-t} \le \rho}  \invmodcont(\Ds;s) & \text{ if } \norm{t}_0 \le k \\
    \infty  & \text{ if } \norm{t}_0 > k
    \end{cases}
\end{align*}

\newcommand{\mechssp}{M^\rho_{\mathsf{sp}}}
Then, our sparse-variant of the inverse sensitivity mechanism $\mechssp$ applies the exponential mechanism with $\lens$ as the score function,
\begin{equation}
\label{eq:sparse-inv-mech}
    \pi_{\Ds}(t) = \frac{ e^{-\lens(\Ds;t) \diffp/2}  }{ \int_{s \in \R^d}  e^{-\lens(\Ds;s) \diffp/2}  ds}
\end{equation}

We have the following upper bound for this mechanism.
\begin{theorem}
  \label{thm:ub-cont-sparse}
  Let $\func: \domain^n \to \range$ where $\range = \{ v \in \R^d: \norm{v} \le R\}$ such that $\norm{\func(\Ds)}_0 \le k$ for all $\Ds \in \domain^n$ .
  Then for any $\Ds \in \domain^n$, and $\beta >0$, with probability at least $1 - \beta$, the (sparse) inverse-sensitivity mechanism~\eqref{eq:sparse-inv-mech} with norm $\norm{\cdot}$ has error
    \begin{equation*}
      L\left(\smmech(\Ds;f),f(\Ds)\right) \le \modcont_\func^L\left(\Ds;K\right)+c_L\rho,
  \end{equation*}
  where 
  \begin{equation}\label{eq:corruptions-sparse}
  K\geq \frac{2\left(k \left(\log(ed/k) + \log(R/\rho+1) \right) +\log\frac{1}{\beta}\right)}{\diffp}
  \end{equation}
  and $\modcont_\func^L\left(\Ds;K\right)=\sup_{\Dsc:\dham(\Ds,\Dsc)\leq K} L\left(f(\Ds),f(\Dsc)\right)$ denotes the local modulus of continuity of $f$ at $\Ds$ with respect to loss function $L$.
\end{theorem}

\begin{proof}
    The proof follows similar steps to the proof of~\Cref{thm:ub-cont-loss}.
    We define the good set of outputs $A = \{ t \in \R^d: \lens(\Ds;t) \le K\}$, where $\lens(\Ds;t)$ is defined with respect to norm $\norm{\cdot}$ and $K\in\naturals$. By the definition of sparse inverse sensitivity, for any $t\in A$, $t$ is $k$-sparse and there exists $s \in \range$ with $\invmodcont(\Ds;s) = K$ and $\norm{s-t} \le \rho$. We will show that $\Pr[\smmech(\Ds) \notin A] \le \beta$ for sufficiently large $K$. This implies the desired upper bound as we have that for $t \in A$
    \begin{align*}
   L(t,\func(\Ds))
         & \le L(t,s)+L(s,\func(\Ds)) \tag{by triangle inequality}\\
         & \le c_L\norm{t-s} + \modcont_\func^L(\Ds;K) \tag{since $s\in\ball(R)$, $\norm{t-s}\leq \rho$, so $s,t\in\ball(R+\rho)$}\\
         & \le c_L\rho + \modcont_\func^L(\Ds;K).
    \end{align*}
    Now we upper bound $\Pr[\smmech(\Ds;\func) \notin A]$.  First, note that $\sminvmodcont(\Ds;u) = 0$ for $u$ such that $\norm{u-\func(\Ds)} \le \rho$. This implies that for any $t$ such that $\sminvmodcont(\Ds;t) \ge K$, the density is upper bounded by $0$ if $\norm{t}_0 >k$, and otherwise,
    \begin{equation*}
    \pi_{\Ds}(t) \le \frac{e^{-K\diffp/2}}{\int_{u : \norm{u}_0 \le k, \norm{u-\func(\Ds)} \le \rho} du}
    \end{equation*}
    Overall, this implies that 
    \begin{align*}
    \Pr[\smmech(\Ds;\func) \notin A] 
        & \le e^{-K\diffp/2} \frac{\int_{u \in \R^d: \norm{u}_0 \le k, \norm{u} \le R + \rho}  du}{\int_{u  \in \R^d : \norm{u}_0 \le k,  \norm{u-\func(\Ds)} \le \rho} du} \\
        & \le e^{-K\diffp/2} \choose{d}{k} \frac{\int_{u \in \R^k: \norm{u} \le R + \rho}  du}{\int_{u \in \R^k :   \norm{u} \le \rho} du} \\
        & \le e^{-K\diffp/2} \choose{d}{k}  (R/\rho + 1)^k \\
        & \le e^{-K\diffp/2} (ed/k)^k  (R/\rho + 1)^k.
    \end{align*}
    \lz{do we need to explain the second inequality more?}
    Setting $K\geq \frac{2k (\log(ed/k) + \log(R/\rho+1)) +2\log(1/\beta)}{\diffp}$, 
    we get that $\Pr(\smmech(\Ds;\func) \notin A)  \le \beta$.
\end{proof}

Using this mechanism, we now have the following transformation from robust-to-private for sparse estimators.
\begin{theorem}[Robust-to-private for sparse estimators]
\label{thm:high-dim-rob-to-priv-sparse}
  Let $\Ds = (S_1,\dots,S_n)$ where $S_i \simiid P$ such that $\mu(P) \in \R^d.$
  Let $\diffp, \beta \in (0,1)$. Let $L:(\R^d)^2\to \R$ be a loss function which satisfies the triangle inequality.
  Let $\A_\text{rob} : (\R^d)^n \to \{ t \in \R^d: \norm{t} \le R\}$ be a (deterministic) $(\tau,\beta,\alpha)$-robust algorithm with respect to $L$ such that $\norm{\A_\text{rob}(\Ds)}_0\le k$ for all $\Ds$. Let $\alpha_0\leq \alpha$. Suppose $n$ is such that the smallest value $\tau$ satisfying~\Cref{eq:fraction-sparse} is at most $1$. Suppose $\forall u,v\in\ball(R+\alpha_0) ~ L(u,v)\leq c_L\norm{u-v}$ for some constant $c_L$. 
  If 
  \begin{equation}\label{eq:fraction-sparse}
  \tau \geq \frac{2\left(k \left(\log(ed/k) + \log(R/\alpha_0+1) \right) +\log\frac{1}{\beta}\right)}{n \diffp}
    \end{equation}
    then Algorithm $\mechssp(\Ds;\A_\text{rob})$ with $\rho = \alpha_0$ in norm $\norm{\cdot}$ is $\diffp$-DP and, with probability at least $1-2\beta$, has error
  \begin{equation*}
      L\left(\mechssp(\Ds;\A_\text{rob}),\mu\right) \le (3+c_L)\alpha=O(\alpha).
  \end{equation*}
\end{theorem}
We leave the proof as an exercise for the reader as it is identical to the proof of~\Cref{thm:high-dim-rob-to-priv-loss-app} using the upper bounds for the sparse variant of the inverse sensitivity mechanism (\Cref{thm:ub-cont-sparse}). 


\section{Useful Facts and Proofs for Applications}\label{app:facts}
\subsection{Linear Algebra Facts and Definitions}
We denote by $\|v\|_{M}=\|M^{-1/2}v\|=\sqrt{v^\top M^{-1} v}$ the Mahalanobis norm of vector $v$ with respect to $M$ for any positive definite matrix $M$. Observe that $\norm{v}_{\id}=\ltwo{v}$.
\begin{proposition}\label{prop:distance-wrt-matrices}
    For positive definite matrices $\Sigma_1, \Sigma_2$, if $\Sigma_1 \preceq \Sigma_2$, then for any vector $v$, $\|v\|_{\Sigma_2}\le \|v\|_{\Sigma_1}$.
\end{proposition}
Let $A\in \R^{d\times d}$. We denote the \emph{spectral norm} of $A$ by $\|A\|_2=\sup\{\|Ax\|_2: x\in\mathbb{R}^d \text{ s.t. }\|x\|_2=1\}$ and its \emph{Frobenius norm} by $\|A\|_F=\sqrt{\sum_{j=1}^d \sum_{i=1}^d |A_{i,j}|^2}$. It holds that $\ltwo{A}\leq \lfro{A}\leq \sqrt{d}\ltwo{A}$.

\subsection{Robustness Guarantee of Tukey Median}
We first state known properties of the Tukey depth for Gaussian datasets. The next proposition relates the Tukey depth of a point to its Mahalanobis distance from the mean (see e.g.\ Proposition D.2 in~\cite{BrownGSUZ21} for a proof). Here, $\Phi$ is the CDF of the univariate standard Gaussian.
\begin{proposition}\label{prop:tukey-cdf}
    For any $\mu, y\in \mathbb{R}^d$ and positive definite $\Sigma$, $T_{\normal(\mu,\Sigma)}(y)=\Phi(-\|y-\mu\|_{\Sigma})$. 
\end{proposition}

The next proposition states the uniform convergence property of Tukey depth. 
It follows from standard uniform convergence of halfspaces~\citep{VapnikC97}, extended to the definition of Tukey depth~\cite{DonohoG92, BurrF17} (see e.g.\ ~\cite{LiuKK021} for a complete proof).
\begin{proposition}[Convergence of Tukey Depth]\label{prop:tukey-convergence}
   Let $\Ds = (S_1,\dots,S_n)$ where $S_i \simiid \normal(\mu,\Sigma)$. There exists constant $C_0$ such that, with probability $1-\beta$, for any $v\in\R^d$, $|T_{\normal(\mu,\Sigma)}(v)-T_{\Ds}(v)|\leq C_0\cdot\sqrt{\frac{d+\log(1/\beta)}{n}}$.
\end{proposition}

\begin{proposition}[Robust Accuracy of Tukey Median, Restatement of~\Cref{prop:accr-tukey}]\label{prop:accr-tukey-app}
Let $\Ds = (S_1,\dots,S_n)$ where $S_i \simiid \normal(\mu,\Sigma)$ such that $\mu \in \ball(R)$ and $\id\preceq \Sigma \preceq \kappa\id$. Let $\beta\in(0,1)$, $\tau\leq 0.05$, and $\alpha_0=C_0\cdot\sqrt{(d+\log(1/\beta))/n}$ as in~\Cref{prop:tukey-convergence}. Suppose $n$ is such that $\alpha_0\leq 0.05$. Let $\alpha=7(\alpha_0+\tau)\leq 1$.
The projected Tukey median algorithm 
$\A_{\text{rob}}(\Ds)=\Pi_{\ball\left(R+\sqrt{\kappa}\right)}(t_m(\Ds))$ 
is $(\tau,\beta,\alpha)$-robust with respect to the Mahalanobis loss.
That is, with probability $1-\beta$, for any $\tau$-corrupted $\Dsc$, such that $\dham(\Ds,\Dsc)\leq n\tau$, it holds that
$
\norm{\A_\text{rob}(\Ds') - \mu}_{\Sigma} \le \alpha.
$
\end{proposition}
\begin{proof}
Let $\Dsc$ be any $\tau$-corruption of $\Ds$, that is, $\dham(\Ds,\Dsc)\leq n\tau$. Observe that $|T_{\Ds}(v)-T_{\Dsc}(v)|\leq \tau$ for any $v\in\R^d$ by the definition of Tukey depth. 
Let $t_m'=\argmax_{v\in\R^d} T_{\Dsc}(v)$ be the Tukey median of the corrupted dataset. We condition on the event that the bound of~\Cref{prop:tukey-convergence} holds, which occurs with probability $1-\beta$. We have that
\begin{align*}
& T_{\normal(\mu,\Sigma)}(\mu)=\frac{1}{2} \tag{by~\Cref{prop:tukey-cdf} since $\Phi(0)=\frac{1}{2}$}\\
& \Rightarrow T_{\Ds}(\mu)\geq \frac{1}{2}-\alpha_0 \tag{by~\Cref{prop:tukey-convergence}} \\
& \Rightarrow T_{\Dsc}(\mu)\geq \frac{1}{2}-\alpha_0-\tau \\
& \Rightarrow T_{\Dsc}(t_m')\geq \frac{1}{2}-\alpha_0-\tau \tag{by definition of $t_m'$}\\
& \Rightarrow T_{\Ds}(t_m')\geq \frac{1}{2}-\alpha_0-2\tau \\
& \Rightarrow T_{\normal(\mu,\Sigma)}(t_m') \geq \frac{1}{2} -2\alpha_0-2\tau \tag{by~\Cref{prop:tukey-convergence}}\\
& \Rightarrow \Phi(-\norm{t_m'-\mu}_{\Sigma}) \geq \frac{1}{2} -2\alpha_0-2\tau  \tag{by~\Cref{prop:tukey-cdf}} \\
& \Rightarrow \frac{1}{2}\mathrm{Erf}\left(\frac{\norm{t_m'-\mu}_{\Sigma}}{\sqrt{2}}\right) \leq 2(\alpha_0+\tau) \tag{since $\Phi(-z)=\frac{1}{2}-\frac{1}{2}\mathrm{Erf}\left(\frac{z}{\sqrt{2}}\right)$}
\end{align*}
It is easy to see that the following bound holds for the error function $0.84z\leq \mathrm{Erf}(z)$ for $z\in[0,1]$ (see e.g.\ Lemma 3.2 in~\citep{CanonneKMUZ20}). It follows that, $\norm{t_m'-\mu}_{\Sigma}\leq \frac{4\sqrt{2}}{0.84}(\alpha_0+\tau)\leq 7(\alpha_0+\tau)$ for $\alpha_0+\tau\leq 1/7$, which holds by assumption. Thus, with probability $1-\beta$, 
$\norm{t_m'-\mu}_{\Sigma}\leq \alpha$, for $\alpha = 7(\alpha_0+\tau)\le 1$. Since we have assumed that $\ltwo{\mu}\leq R$, it follows that $\ltwo{t_m'}\leq \ltwo{\mu}+\ltwo{t_m'-\mu}\leq R+\sqrt{\kappa}\norm{t_m'-\mu}_{\Sigma}\leq R+\sqrt{\kappa}\alpha\leq R+\sqrt{\kappa}$, where the second inequality holds due to~\Cref{prop:distance-wrt-matrices}. Then $t_m'\in\ball(R+\sqrt{\kappa})$ and the projection will not affect the output.
\end{proof}

\end{document}